\documentclass{article}
\usepackage{mathrsfs,amsthm,amsmath,amssymb,bbm,pdfsync,prettyref,hyperref,graphicx}
\usepackage{mathabx}
\usepackage{epstopdf}
\usepackage{array}

\usepackage{fancyhdr,setspace}

\usepackage[marginratio=1:1,height=600pt,width=460pt,tmargin=100pt]{geometry}
\usepackage[compress]{cite}

\usepackage{authblk}
\usepackage{algorithm2e,tikz}
\usepackage{algpseudocode}

\usepackage{hyperref}
\hypersetup{
  colorlinks = false,
}

\usepackage[caption=false]{subfig}

\newcolumntype{C}[1]{>{\centering\arraybackslash$}p{#1}<{$}}

\newcommand{\cB}{\mathcal{B}}
\newcommand{\cC}{\mathcal{C}}
\newcommand{\cV}{\mathcal{V}}

\newcommand{\R}{\mathbb{R}}

\newtheorem{theorem}{Theorem}[section]
\newtheorem{assump}{Assumption}

\newtheorem{lemma}{Lemma}

\newtheorem{proposition}{Proposition}

\begin{document}

\title{Spectral Embedding Norm: Looking Deep into the Spectrum of the Graph Laplacian}

\author[1]{Xiuyuan Cheng}
\author[2]{Gal Mishne\footnote{Email: gmishne@ucsd.edu.}}
\affil[1]{Department of Mathematics, Duke University}
\affil[2]{Hal\i c\i o\u glu Data Science Institute, University of California, San Diego}

\date{}
\maketitle

\begin{abstract}
The extraction of clusters from a dataset which includes multiple clusters 
and a significant background component is a non-trivial task of practical importance.
In image analysis this manifests for example in anomaly detection and target detection.
The traditional spectral clustering algorithm, 
which relies on the leading $K$ eigenvectors to detect $K$ clusters, 
fails in such cases.
In this paper we propose the {\it spectral embedding norm} 
which sums the squared values of the first $I$ normalized eigenvectors, 
where $I$ can be significantly larger than $K$. 
We prove that this quantity can be used to separate clusters from the background in unbalanced settings,
 including extreme cases such as outlier detection.
The performance of the algorithm is not sensitive to the choice of $I$,
and we demonstrate its application on synthetic and real-world remote sensing and neuroimaging datasets.
\end{abstract}


\section{Introduction}

In unsupervised learning and data analysis, one of the most common goals is to group the data points into clusters.
A variant task is to extract interesting clusters from the data when, in practice, data points do not perfectly fall into $K$ clusters.
We consider a non-trivial setting in which 
data consist of not only interesting sub-groups, 
namely ``clusters'', 
but also a large component containing points which are less structured or of less interest, which we call ``background''.
Important examples in imaging data analysis
include image segmentation and saliency detection where the clusters are regions of interest in the image, and the background consists of the rest of the image~\cite{borji2015salient,peng2013}.
Another example is the task of anomaly (or outlier) detection
where anomalous samples (small clusters) in the dataset differ from the normal ones (background)
and indicate that something important has happened or a problem has occurred. 
By the very nature of the problem, most data points belong to the background and only a small fraction of data points are anomalies.
Anomaly detection in images is an important task in a variety of applications such as target detection in remote sensing imagery, detecting abnormalities such as tumors in biomedical imagery and for quality inspection in production lines. An automated solution highlighting only suspicious regions to be reviewed by an expert would save greatly on time.

In theory, one may view the background component as an extra cluster, 
however the unbalanced size of the clusters versus background components poses challenges for traditional clustering methods.
The popular spectral clustering algorithm~\cite{Weiss1999,Ng2002,Shi2000} reduces the dimensionality of the data using a spectral embedding, 
and then performs clustering in the low-dimensional space. 
The method originally proposed to cluster data into $K$ clusters by applying $k$-means to the leading $K$ eigenvectors 
(the low-lying eigenvectors of the graph Laplacian)
computed from an affinity matrix built from the data~\cite{Weiss1999,Ng2002,selftune,Nadler07b}.
A question is then how to set the parameter $K$, 
and this is especially important for exploratory data analysis when the number of clusters underlying the data is not known {\it a priori}.
The traditional solution is to use the spectral gap of the eigenvalues 
to determine $K$ \cite{von2007tutorial}, 
yet in practical settings, such a gap may not exist.
In particular,  when the cluster sizes are unbalanced,
or a large background component is present,
there is no spectral gap after the $K$-th eigenvalue
and
the leading eigenvectors do not localize on $K$ given clusters, but rather tend to be supported 
mostly on the large component due to the slow mixing time of the diffusion process restricted to it.
The unbalanced case of outlier detection is a classical scenario where traditional spectral clustering fails to identify the existing clusters 
\cite{Nadler2006a,zhao2010spectral,xiang2008spectral}.
As has been shown in~\cite{miller2010subgraph,cloninger2015eigenvector,nian2016auto,wu2013spectral} and will be demonstrated, 
the eigenvectors which indicate clusters or outliers may lie deep within the spectrum of the affinity matrix.  
This gives rise to the notion of abandoning the guideline of focusing on $K$ eigenvectors and rather choosing to look deeper into the spectrum
in such settings.

In this paper we consider a cluster-background splitting model of the graph, including anomaly detection as a special case. 
The model is motivated by applications and will be tested on real-world datasets.
We propose a quantity called the {\it spectral embedding norm},
which maps each node in the graph to a positive number, 
and separates clusters from background with a theoretical guarantee. 
The idea is closely related to the ``localization'' pattern of the eigenvectors, namely where they are supported on---either mainly on the cluster block or on the background block---
and this pattern maintains even when the spectral gap vanishes. 
Viewing the affinity matrix as a perturbed one from a baseline affinity where the background and clusters are completely disconnected, 
one can analyze the consequent deformation of the spectrum of the graph Laplacian matrix. 
However, the instability of eigenvectors under the deformation poses difficulty to the use of individual eigenvectors in this environment. 
The {\it spectral embedding norm}, 
on the other hand,
improves the stability by using a summation over multiple eigenmodes and 
provides guaranteed detection of the clusters by simple thresholding. 
The algorithm involves a parameter which is the number of eigenvectors summed over,
and the performance is not sensitive to the parameter choice.

Our result thus provides a way to go beyond dominating eigenvectors of the graph Laplacian to unbalanced data clustering tasks with theoretical verification. 
It suggests that, in the presence of cluttered background samples, 
it is beneficial to look deep into the spectrum to identify important and subtle structures.
By providing a simple measure by which to separate clusters from background, 
the method allows the application of signal processing and machine learning methods to the cluster samples only without contamination from the background. 
Our analysis spans a setting of multiple clusters against background to the extremely unbalanced case of outlier detection.

In the rest of the paper,
we review more related literature before ending the current section. 
The {\it spectral embedding norm} is introduced in Section \ref{sec:2}, with an illustrative example to show the main idea. 
The theoretical result is presented in Section \ref{sec:3},
experiments on synthetic and real-world image datasets in Section \ref{sec:4}, 
proofs in Section \ref{sec:proof}, 
and further remarks in the final section. 

{ \bf Notations.}
$|\cdot |$ stands for the cardinal number of a set. $A^c$ means the complement of a set $A$.

\subsection{Related works}
As spectral clustering and variants have been intensively studied in literature, 
we list the most relevant works to our problem. 
 
The spectral embedding for clustered data has been previously analyzed in many places. 
While Schiebinger et al.~\cite{schiebinger2015geometry} analyzed a nonparametric mixture model to show that under certain conditions 
the embedded points lie in an orthogonal cone structure and $k$-means succeeds in clustering the data, 
Nadler and Galun~\cite{Nadler2006a} showed that even for well separated Gaussians the top $K$ eigenvectors do not necessarily localize on $K$ clusters, 
based on the analysis of a diffusion process in a multi-well potential~\cite{Nadler07b, Singer:2009}.  
Different approaches, e.g. \cite{selftune, damle2016robust}, 
attempted to align the eigenvectors axes with the different clusters and improve the robustness of cluster identification.
Zelnik-Manor and Perona~\cite{selftune} proposed to estimate the number of clusters from the eigenvectors instead of from the spectral gap,
and empirically demonstrated improved performance when a background cluster is present.
Damle, Minden and Ying~\cite{damle2016robust} considered the case of balanced block-like affinity matrix. 
The {\it embedding norm} studied in the current work differs from the above approaches, 
and it involves a simple algorithm with theoretical guarantee under the specified settings.

Spectral embeddings have been used for anomaly detection in several modified ways:
based on the first non-trivial eigenvector of an affinity matrix~\cite{ide2004eigenspace,nian2016auto}, 
eigenvector selection~\cite{miller2010subgraph,wu2013spectral}, 
out-of-sample extension~\cite{Mishne2013,aizenbud2015pca,Mishne2017,mishne2017iterative}, 
the algebraic structure of the weighted magnitude sum of Laplacian eigenfunctions
\cite{CHENG201848} and multiscale constructions of spectral embeddings~\cite{Rabin2010,Mishne2013}, and usually requiring tuning of multiple parameters.
For a general review of anomaly detection methods, the interested reader is referred to~\cite{chandola2009anomaly,goldstein2016comparative,akoglu2015graph,ehret2019image}.
In particular, 
``eigenvector selection'' 
has been proposed to determine eigenvectors that localize on clusters or specifically on anomalies, 
using unsupervised spectral ranking~\cite{ nian2016auto}, 
kurtosis~\cite{wu2013spectral}, relevance learning~\cite{xiang2008spectral}, entropy~\cite{de2002eigenfunction,zhao2010spectral}, 
the L1 norm~\cite{miller2010subgraph}, 
tensor product~\cite{cucuringu2011localization}  and local linear regression~\cite{DSILVA2018}.
Both~\cite{xiang2008spectral,zhao2010spectral} proposed calculating $Km$ eigenvectors where $m>1$ and then select informative eigenvectors.
Wu et al.~\cite{wu2013spectral} analyzed the adjacency matrix of a graph, 
while Miller, Bliss and Wofle~\cite{miller2010subgraph} considered the modularity matrix of a graph.
In the current paper, we analyze the spectrum of the 
(normalized) random-walk Laplacian matrix which has a more stable spectrum with finite samples \cite{von2008consistency}.

The proposed notion of {\it embedding norm} 
formally resembles the probability amplitude in quantum mechanics,
which sums squared modulus of the low-energy wave functions. 
It is also similar to the {\it leverage score} in statistics, 
defined to be the squared sum of principal components,
which has been used as an indicator of outlier samples 
in linear regression and other statistical applications~\cite{hoaglin1978hat}.
The generalized form of the embedding norm 
with exponentially decaying weights has been previously suggested 
as a tool to identify salient features in shape analysis, called {\it Heat kernel signature} \cite{sun2009concise}.
However, 
the settings in the previous statistical and computer graphical studies are different from our consideration of the cluster-background separation,
and particularly, the equal weights on the truncated sum over eigenvectors
which more resembles probability amplitude 
has its own motivation, see more in the last section.

%
\section{Spectral Embedding Norm}
\label{sec:2}

Given $n$ data points in the feature space,
an undirected weighted graph can be constructed
which has $n$ nodes, denoted by $\cV$,
and the weight on edge $(x,y)$ is the affinity between nodes $x$ and $y$ denoted by $W(x,y)$.
$W$ is an $n$-by-$n$ real-symmetric matrix of non-negative entries $W(x,y) = W(y,x) \ge 0$, called the graph affinity matrix.
In applications,
$W$ is built as a pairwise affinity between data points in a feature space, e.g., 
$W(x,y) = k(x,y)$, where $k$ is a symmetric kernel function applied to the feature vectors of data points $x$ and $y$. 
In our analysis we assume that $W$ has been constructed. 

\subsection{Cluster-Background splitting in the graph}
Suppose that $\cV$ can be divided into two disjoint subsets, background and clusters,
denoted by $\cB$ and $\cC$ respectively. 
The typical scenario which we consider is when 
data points in $\cC$ are concentrated in the feature space and well-clustered into $K$ sub-clusters,
whereas those in $\cB$ can be ``manifold-like'' and spread over the space. The precise assumptions will be formulated in terms of 
the graph-Laplacian spectra constrained to the subgraphs of $\cC$ and $\cB$
(see Assumption \ref{assump:A1}). 
We also assume the connections between $\cC$ and $\cB$ are weak.
As a result, the submatrix of $W$ constrained to $\cC$ is close to having $K$ blocks,
and is almost separated from the submatrix of $\cB$. 
Define matrix $W_0$ by removing all the connections between $\cB$ and $\cC$ from $W$,
i.e. $W_0$ is a block-diagonal matrix consisting of two blocks of $\cC$ and $\cB$ respectively. 
We introduce a pseudo-dynamic parametrized by time $t$ as 
\begin{equation}
W(t)=W_{0}+t E,
\quad t \in [0,1]
\label{eq:deformW}
\end{equation}
so that $W(0)=W_{0}$ and $W(1)=W$.
To simplify the analysis, we assume that the $K$ sub-clusters are of equal size, that is,
\[
|\cC| = \delta |\cV|, 
\quad |\cB| = (1-\delta)|\cV|, 
\]
and each of the $K$ clusters in $\cC$ has $\frac{\delta |\cV|}{K}$ nodes. 
The result extends to the unequal-size case.

\subsection{Graph Laplacian and embedding norm}
Consider the normalized random-walk graph Laplacian of $W$
\[
L=I-D^{-1}W:=I-P, 
\]
where $D$ is a diagonal matrix defined by $D_{ii} =\sum_j W_{ij}$,  
and $P$ is a Markov matrix. We shall see that $D$ is always invertible.
$P$ is similar to $D^{-1/2}WD^{-1/2}$ which is real-symmetric, thus $P$ is diagonalizable and has $n$ real eigenvalues.
Let
\begin{equation}
\label{eq:evec_ortho}
P \psi_{k}=\lambda_{k}\psi_{k},
\quad
k = 1,\cdots, n,
\quad
\psi_{k}^{T}D\psi_{j}=\delta_{kj},
\end{equation}
where $\{\lambda_{k}\}_k$ are the eigenvalues of $P$, $\{\psi_{k}\}_k$ are the corresponding right eigenvectors and  $\delta_{kj} = 1$ when $k=j$ and 0 otherwise.
Given that $|\lambda_{k}| \le 1$ (Perron-Frobenius theorem),
and when $W$ is a positive-definite kernel matrix then all the eigenvalues are between 0 and 1.
The largest eigenvalue of $P$ is $1$ and the associated eigenvectors is the constant vector.
Note that time dependence is omitted in the above notations:
As we introduced the deformation of $W$ in \eqref{eq:deformW}, 
$D$, $P$ and consequently $\psi_k$ and $\lambda_k$ also depend on $t$.
We assume that at $t=1$ the eigenvalues are sorted to be decreasing,
and for other $t$, the indexing $k$ is arranged so that 
$\psi_k$ and $\lambda_k$ are differentiable with respect to $t$ \cite{kato2013perturbation}.

The {\it spectral embedding norm} of every node $x \in \cV$ is defined to be
\begin{equation}
\label{eq:def-Sx}
S(x)=\sum_{k\in I}\psi_{k}(x)^{2}
\end{equation}
where $I$ is a subset of  the eigenvalue indices $ \{1, \cdots, n\}$.
$S(x)$ is the (squared) Euclidean norm of 
the embedded vector of $x$ 
in the spectral embedding space using eigenvectors of indices in $I$. 
When needed, we include the time dependence in the notation written as $S(x,t)$, $t \in [0, 1]$.
A typical choice of $I$ is $I=\{1,\cdots,|I|\},$
when the eigenvalues are sorted to be descending.
The cardinal number $|I|$ is a parameter of the method
and in the scenario of outlier detection it is typically larger than $K$.
In practice, estimates of $K$ can be used if $K$ is not known.
We will show that the result is not sensitive to the choice of $|I|$, in analysis and experiments. 

The embedding norm $S$ is able to separate $\cC$ from $\cB$ by a provable margin under certain assumptions (Theorem \ref{thm:S-sepa}). 
A general weighted form of $S$ can be introduced and the result extends directly.
This is naturally related to the {\it diffusion distance} \cite{coifman2006diffusion},
and we will explain more on this in the last section. 

\begin{figure}[t]
 \subfloat[]{\includegraphics[height=0.25\linewidth]{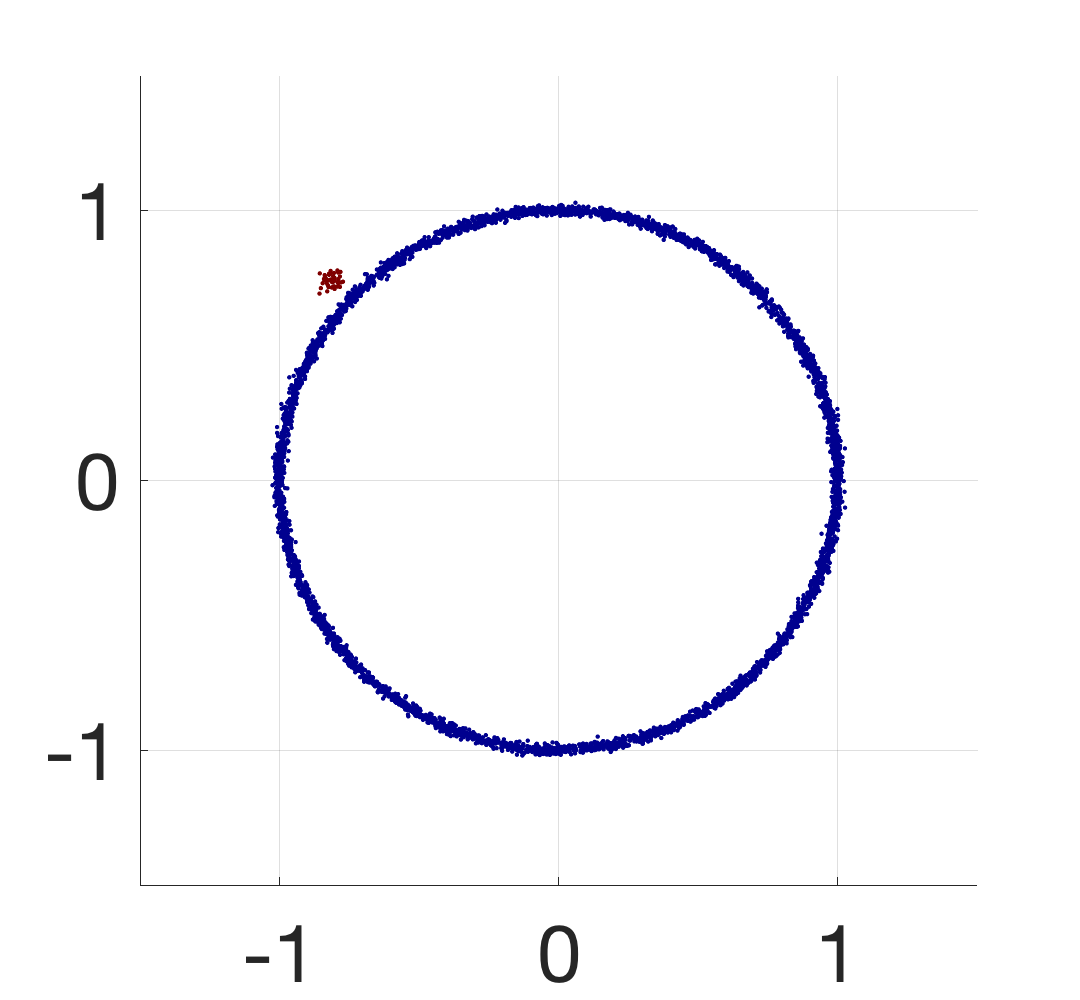}} ~~~~~~
 \subfloat[]{\includegraphics[height=0.25\linewidth]{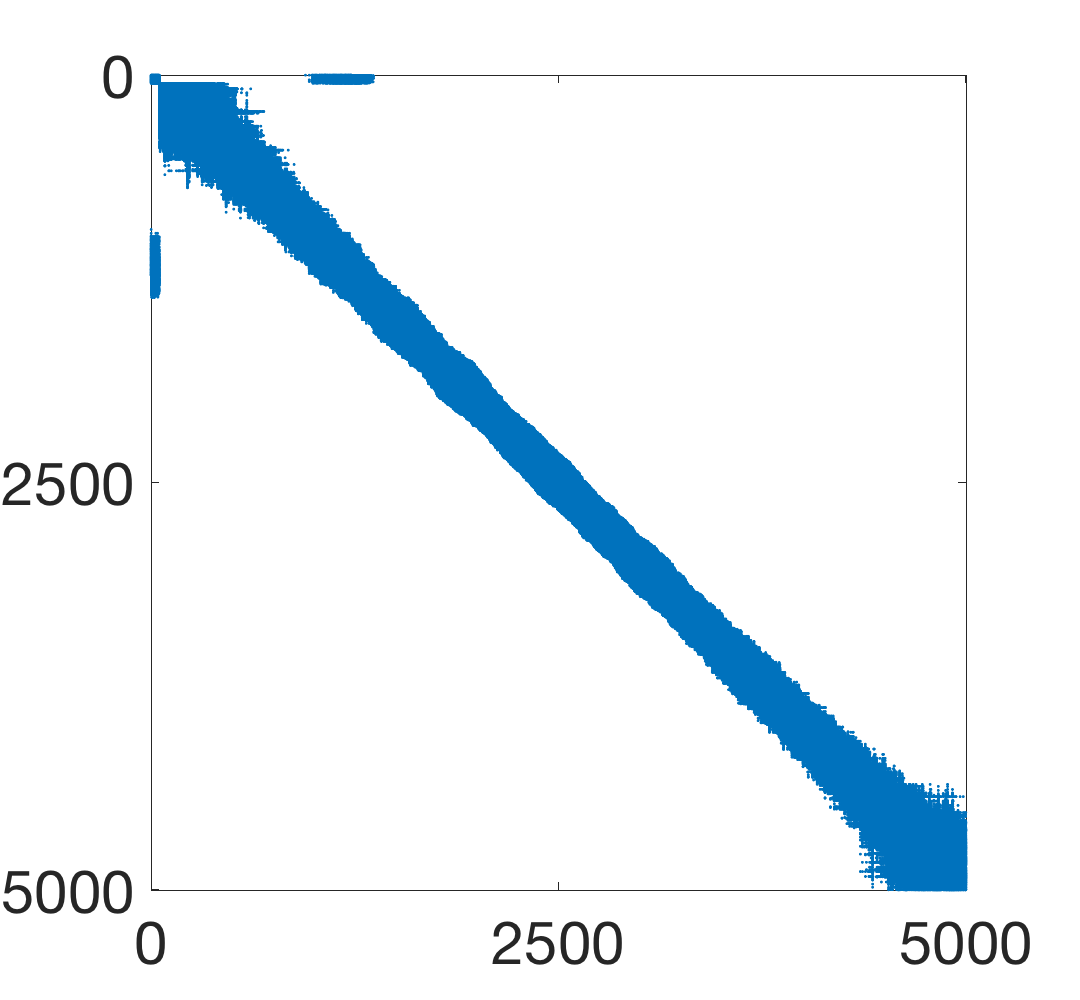}} ~~~~~~
 \subfloat[]{\includegraphics[height=0.25\linewidth]{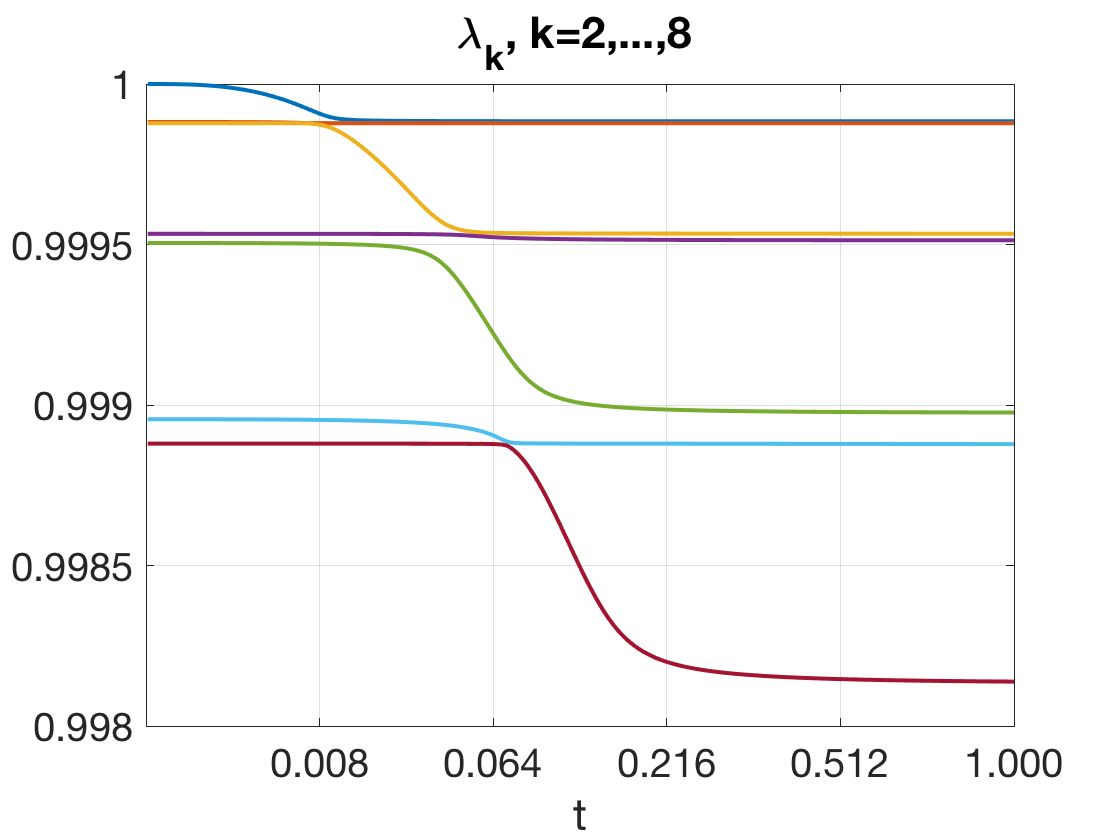}}
 \\
 \makebox[1\textwidth][l]{
  \hspace{-0.8cm}
 \subfloat[]{\includegraphics[height=0.27\linewidth]{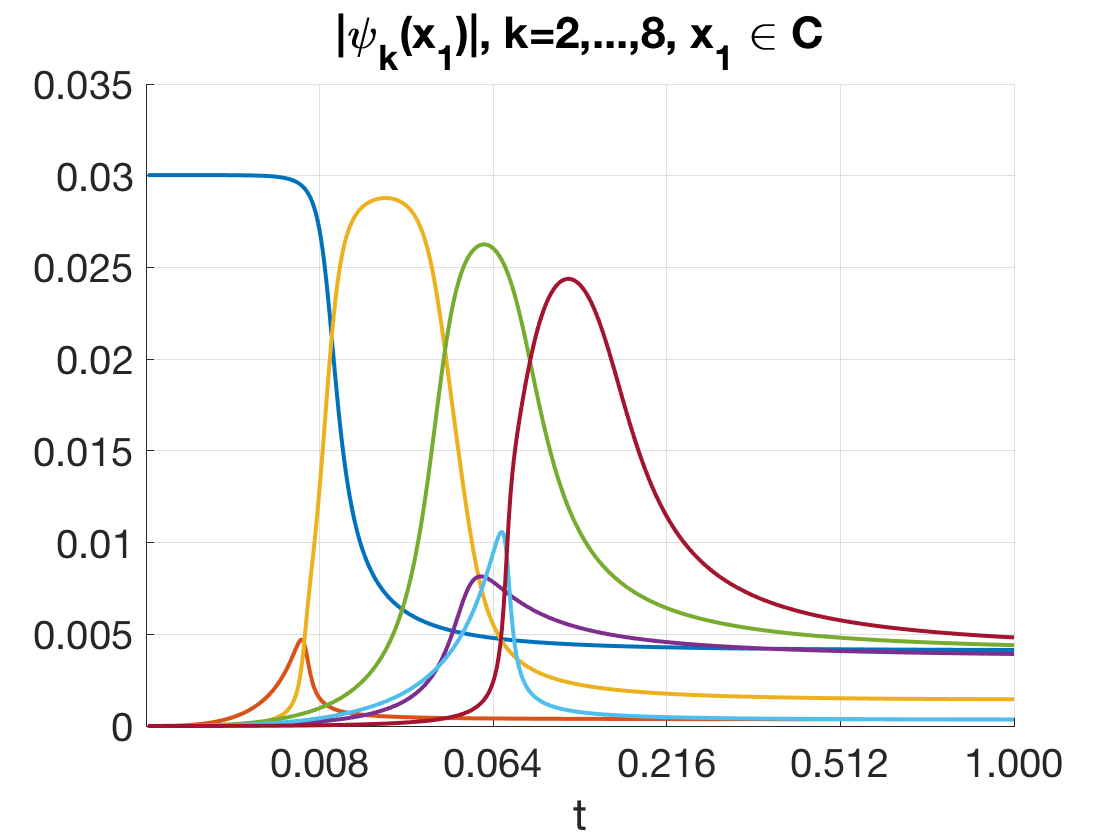}}
  \subfloat[]{\includegraphics[height=0.27\linewidth]{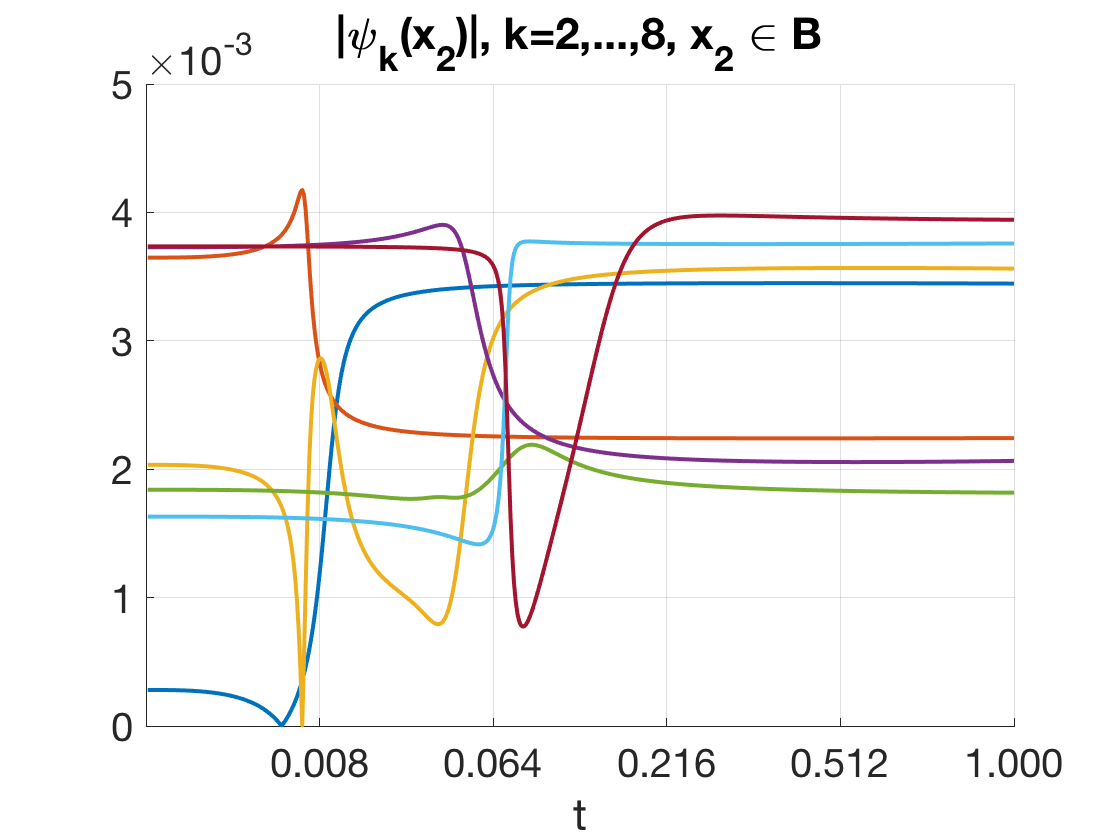}}
  \subfloat[]{\includegraphics[height=0.27\linewidth]{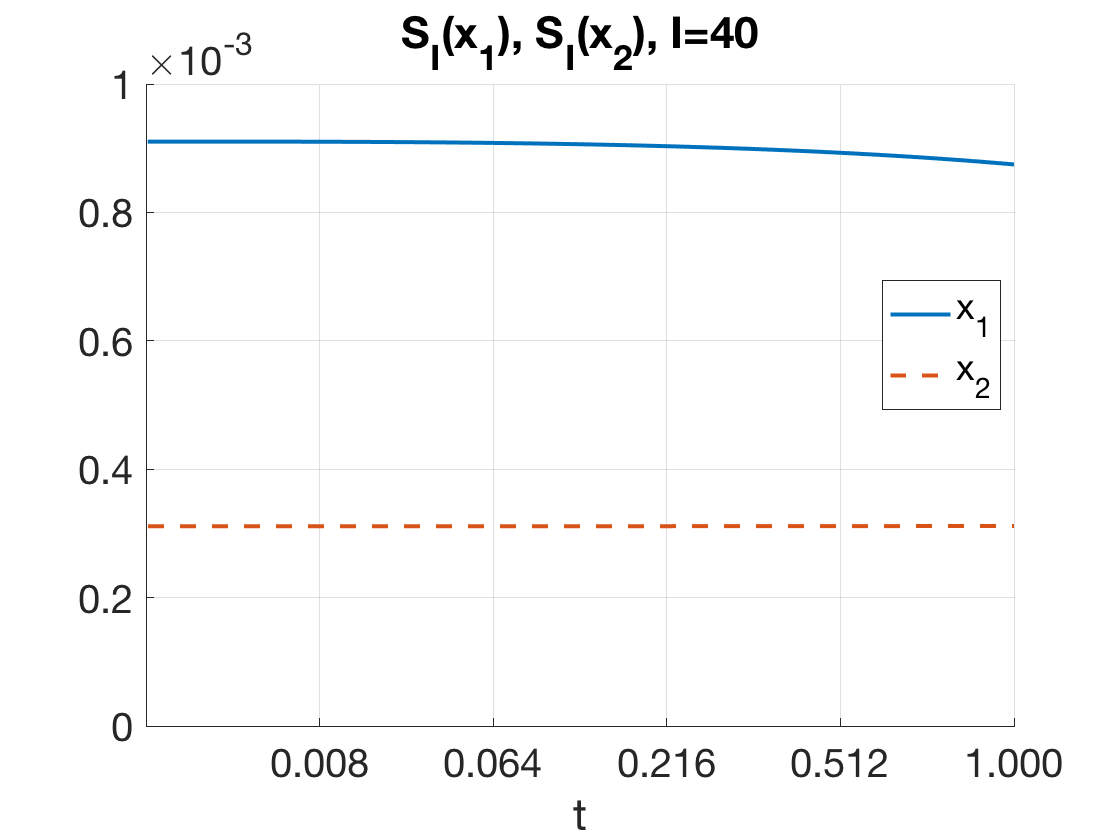}}
  }
\vskip -0.1 in
\caption{ 
\label{fig:toy1}
Plots of eigenvalue and eigenvectors of $P=D^{-1}W$ over time. 
(A) 
$n=5000$ data points in $\R^2$ sampled on ${\cB} \cup {\cC}$, where points in ${\cB}$ lie close to a circle (blue) and
points in ${\cC}$ form the small cluster lying close to the circle (red).
 In this case, $K=1$, $\delta = 0.01$.
 (B) 
 The affinity matrix $W(t)$ at $t=1$, c.f. \eqref{eq:deformW}.
 (C) Plot of the first 8 eigenvalues as $t$ increases from 0 to 1 (excluding $\lambda_1 =1$). 
(D) 
The absolute values of the associated first 8 eigenvectors at $x_1 \in \cC$ over time.
(E) 
Same plot at $x_2 \in \cB$.
(F) The values of the embedding norm $S(x)$ defined in \eqref{eq:def-Sx} 
at $x_1$ and $x_2$ over time, where $|I| = 40$. 
}
\end{figure}

\subsection{A prototypical toy example}\label{subsec:toy}
The prototypical scenario which motivates the proposed method is illustrated in the toy example in Figure \ref{fig:toy1}.
(A) shows data points in $\R^2$ consisting of two groups:
a large group, denoted by $\cB$, which lie close to the unit circle (blue) 
and a small one, denoted by $\cC$, which form a small cluster lying close to the circle (red), and
$\frac{|\cC|}{|\cV|}  = \delta =0.01$.
Variations of this model involve multiple sub-clusters in $\cC$ 
(see Figure \ref{fig:toy2-exp} and the applications on real-world data),
and the qualitative picture is the same. 
The affinity matrix $W$ built from the data is shown in (B).
The first few eigenvalues, evolving over time $t$, are shown in (C),
and the associated eigenvectors evaluated on two nodes, one in $\cC$ and one in $\cB$,
are plotted in (D) and (E).
We sort the eigenvalues of $P$ from large to small, and the first eigenvalue is always 1. 
The {\it embedding norm} $S(x)$ takes
the squared-sum of the first 40 eigenvectors on each node 
(c.f. \eqref{eq:def-Sx}) and is plotted in (F) over time. 
The figure demonstrates that 
\begin{enumerate}
\item 
Though two blocks $\cC$ and $\cB$ exists in the graph,
there is no clear eigen-gap between the second and third eigenvalues.
Actually, the leading eigenvalues are all very close to 1 throughout time $5$ (the eighth eigenvalue is greater than 0.998).
\item 
While the second eigenvector $\psi_2$ distinguishes $\cC$ at short time $t$,
once $t$ is greater than 0.01 $\psi_2$ fails to indicate the cluster $\cC$.
(D) shows the value on one node and it is typical for the value of $\psi_k$'s on $\cC$.
The transition actually happens when the initial gap between $\lambda_2$ and $\lambda_3$ almost vanishes. 
However, the high-index eigenvectors may take large value on $\cC$
(the eigth eigenvector starts to take large magnitude on $\cC$ around $t=0.1$, 
and the trend of high-indexed eigenvectors localizing on $\cC$ continues, which is not shown).
This is evident by $S(x)$ consistently distinguishing $\cC$ from $\cB$ over time, as shown in (F). 
\end{enumerate}

This suggests that when the leading eigenvectors fail to identify the cluster $\cC$, 
the information of the location of $\cC$ may be contained in higher-indexed eigenvectors,
and looking deep into the spectrum may be helpful. 
However, the selection of informative' eigenvectors is generally a challenging problem.
In particular, as shown in (D)(E), 
the deformation of eigenvectors is not stable when eigenvalues get close,
which makes it difficult to study them individually. 
Instead, the embedding norm we proposed varies smoothly over time and preserves a gap between $\cC$ and $\cB$.
As a result,  one can detect $\cC$ from $\cB$ by thresholding the value of $S$ at $t=1$. 
Note that 40 eigenvectors are used in the summation, which is much larger than 2.
We will justify this improved stability in the analysis.

%
\section{Theoretical Analysis of Cluster Detection}
\label{sec:3}

At $t=0$ in \eqref{eq:deformW}, 
the matrix $W_{0}$ has a two-block structure, and the spectrum of the graph Laplacian of $W_{0}$
also splits into two groups, one residing on $\cC$ and the other on $\cB$ respectively. 
However, as $t$ increases,
interactions among the eigenvectors develop and the perfect splitting pattern is no longer preserved.
The {\it embedding norm} varies more stably than individual eigenvectors over time,
and serves as a measure by which to separate $\cC$ from $\cB$ up to time $t=1$.

\subsection{Initial separation by $S$ and assumptions}
Since we will use $S$ to separate $\cB$ and $\cC$, we need it to do so at least at $t=0$
when the two blocks $\cB$ and $\cC$ are perfectly separated by removing all the edges connecting them.
Note that this does not necessarily happen unless certain assumptions are made:
Because the eigenvalues of the $\cB$ block can be close to 1
and the clustering in the $\cC$ block may not be perfect, 
the first $|I|$ eigenvectors may be supported either on $\cC$ or on $\cB$,
and there is generally no guarantee that the squared sum \eqref{eq:def-Sx} will distinguish the two blocks.
We make the following two assumptions: 

(1) At $t=0$, the eigenvectors in $I$ which are supported on $\cB$
are sufficiently delocalized (``flat'') 
and those on $\cC$
are close to the well-clustered case;

(2) The fraction $\delta$ of $|\cC|$ is sufficiently small so that the eigenvectors on $\cC$ are of sufficiently larger magnitude than those on $\cB$,
due to the eigenvector normalization \eqref{eq:evec_ortho}.
The precise condition depends on the choice of $|I|$, the node degrees and so on.

We denote the volume of set $A$ at time $t$ by $\nu(A,t)$ defined as the sum of the degrees at time $t$,
\begin{equation}
\label{eq:def-nu}
\nu(A,t) = \sum_{x \in A} d(x,t),
\quad
d(x,t) = \sum_{y\in \cV } W(x,y;t).
\end{equation}
We also define lower and upper bounds
\begin{equation}\label{eq:def-d0-1}
\underline{d_0} := \min_{x \in \cV } d(x,0),
\quad
\overline{d_0} := \max_{x \in \cV } d(x,0),
\end{equation}
and assume that $\underline{d_0} > 0$.
By construction \eqref{eq:deformW},
the degree $d(x,t)$ of any node monotonically increases over time. 
Thus $\underline{d_0}$ is the universal degree lower-bound:
\begin{lemma}\label{lemma:lower-bound-d}
For all $x\in \cV$ and all $0 \le t \le 1$,
$
d(x,t) \ge \underline{d_0} > 0.
$
\end{lemma}

\begin{figure}[t]
\centering{
\includegraphics[width=0.75\linewidth]{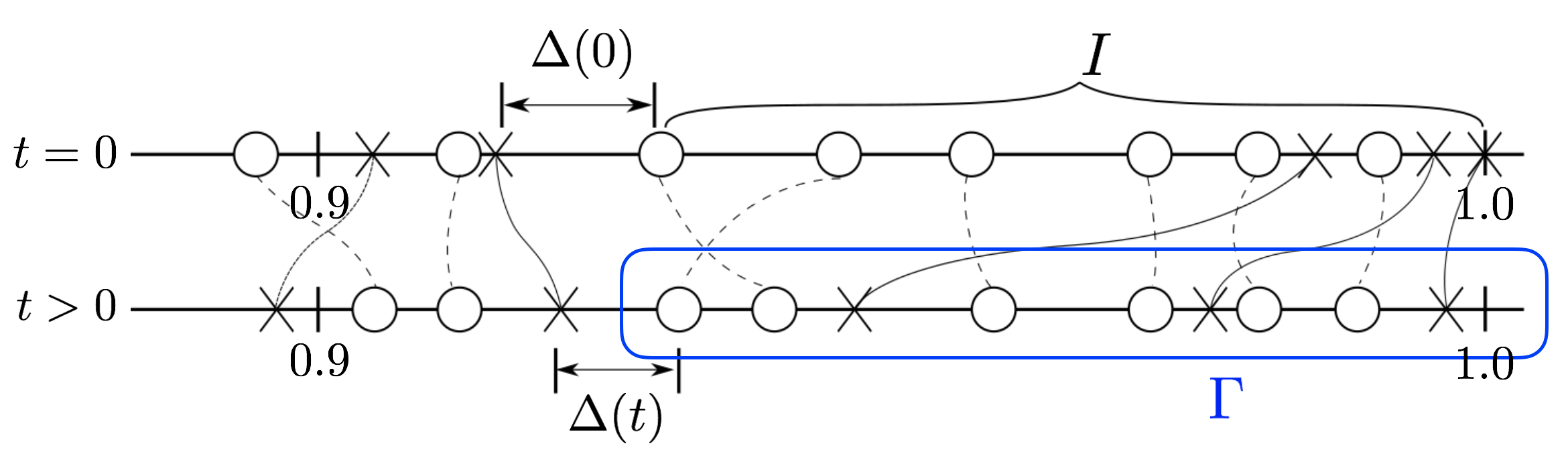}
}
\caption{
Diagram showing the evolution of eigenvalues of the Markov matrix as $W(t)$ changes over time as in \eqref{eq:deformW}.
At $t=0$, circles indicate $\cB$ eigenvalues, and crosses indicate $\cC$ ones. 
In this example, $K=3$, and $|I| = 10$.
Eigenvalues of the $\cB$-submatrix are shown in circles,
and those of $\cC$-submatrix in crosses.
Note that $\cB$ can have eigenvalues close to 1 even at $t=0$.
As $t$ increases, at most time the eigenvalues are all of multiplicity one.
Eigen-crossings may happen within $I$ and $I^c$
but not in between, and the $I$-spectral gap denoted by $\Delta(t)$ is preserved (Proposition \ref{prop:I-gap-preserve}). 
The differential equation \eqref{eq:S-evolve} is obtained by the contour integral $\Gamma$,
which exists for all $t$ due to positive $\Delta(t)$ . 
}
\label{fig:diag1}
\end{figure}

At $t=0$, since the affinity matrix decomposes into two separated blocks $\cB$ and $\cC$, so do the eigenvectors.
We call them initial eigenvectors,
and the set of eigenvectors which are only supported on $\cB$ 
are called the $\cB$-eigenvectors, denoted by $\Psi^{\cB}$,
and similarly for $\cC$-eigenvectors and $\Psi^{\cC}$. 
The assumption on these eigenvectors and the index set $I$ is the following:

\begin{assump}[$\cB$ and $\cC$-eigenvectors]
\label{assump:A1}
At $t=0$,

(a) The index $I$ includes 
$K$ $\cC$-eigenvectors
and $|I|-K$ $\cB$-eigenvectors.

(b) 
Each of the $K$ eigenvectors in $I \cap \Psi^{\cC}$
(up to a $K$-by-$K$ rotation of these $K$ vectors)
is associated with one of the $K$ clusters in the following sense:
There exists $ 0 \le \varepsilon_1 < 1 $, 
and
for each $\psi \in I \cap \Psi^{\cC}$,
there is a unique $j$, $1 \le j \le K$, s.t.
\begin{align}
\frac{1-\varepsilon_1}{ \nu( C_j ,0) }
&\le 
\psi(x)^2
\le
\frac{1+\varepsilon_1}{ \nu( C_j ,0) },
\quad \forall x\in C_j, 
\label{eq:psi-C-bound-1}
\\
\psi(x)^2
&\le 
\frac{\varepsilon_1}{ \nu( C ,0) },
\quad \forall x \in C \backslash C_j,
\label{eq:psi-C-bound-2}
\end{align}

(c) There exists $\varepsilon_2 \ge 0 $, 
s.t. for any $\psi \in I \cap \Psi^{\cB}$, 
\[
\psi(x)^2 
\le
\frac{1+\varepsilon_2}{ \nu(B,0) }, 
\quad \forall x\in B.
\]

\end{assump}

The above assumption, while appearing to be complicated,
poses only generic conditions on the subgraphs $\cB$ and $\cC$:

In the perfectly separated case the largest $K$ $\cC$-eigenvalues 
are 1, and the $(K+1)$-th one is strictly less than 1 
and depends on the mixing time of the Markov chain within each cluster.
This spectral gap is usually significant
since we primarily work with a well-clustered $\cC$ 
which takes a small fraction of nodes and is localized in the graph,
e.g., $\cC$ is an outlier cluster, or several localized regions of interest.
As a result, even when the clustering is not perfect,
the $(K+1)$-th $\cC$-eigenvalue is still sufficiently far away from the first $K$ ones,
and they can be excluded from the index set $I$, since $I$ selects the largest $|I|$ eigenvalues.
This fulfills (a).

If the $K$ clusters in $\cC$ are perfectly separated, 
one can verify that $\varepsilon_1 = 0$ in (b).
Thus (b) holds when $\cC$ is not far from being well-clustered.

Assumption \ref{assump:A1}(c) requires that the eigenvector $\psi$ is sufficiently delocalized, or ``flattened'' on $\cB$:
Recall that~(\ref{eq:evec_ortho}) 
\[
\sum_{x \in \cB} \psi(x)^2 d(x,0) = 1,
\]
and the first eigenvector (associated with eigenvalue 1) takes the constant value $\psi(x)^2 = \frac{1}{\nu(B,0)}$.
If all the other eigenvectors are flattened, then (c) holds with some small $\varepsilon_2$.
The delocalization widely applies when $\cB$ are built from data vectors lying on certain regular manifolds:
assuming that the discrete eigenvectors well approximate the continuous limits 
which are eigenfunctions of the manifold Laplacian, 
the delocalization of the former inherit from that of the latter (Quantum Ergodicity Theorem \cite{zworski2012semiclassical, zelditch2008local}).
When the spectral convergence is poor, the finite-sample effects may create some localized pattern in the ``noisy'' eigenvectors,
however, since $|I|$ is typically a small number compared to $n$,
we assume that the selected $\cB$-eigenvectors are sufficiently close to the population ones.

The second assumption is on the proportion of cluster nodes:
Recall that $\delta = \frac{|\cC|}{ |\cV|}$,

\begin{assump}\label{assump:A2}
The constants $\delta$, $|I|$ and $K$ satisfy that
\begin{equation}\label{eq:assumpA2}
\frac{ \delta}{ 1-\delta } 
	\frac{ |I| - K}{ K }
<
\frac{ \underline{d_0} (1-\varepsilon_2) }{\overline{d_0} (1 + \varepsilon_1)},
\end{equation}
where $\varepsilon_1$, $\varepsilon_2$ are as in Assumption \ref{assump:A1}.
\end{assump}

Prototypical cases where Assumption \ref{assump:A1} and \ref{assump:A2} are satisfied 
include the example in Section \ref{subsec:toy} (Figure \ref{fig:toy1}) and in Section \ref{subsec:4.1} (Figure \ref{fig:toy2-exp}, \ref{fig:toy1-exp}).
Theoretically,
the above two assumptions guarantee that the embedding norm $S(x,0)$ separates the blocks $\cC$ and $\cB$ at time $t=0$,
together with an upper bound of $S(x,0)$ over $\cV$:

\begin{proposition}[Initial separation by $S(x)$]
\label{prop:init-sepa}
Under Assumption \ref{assump:A1}, at time $t=0$,
\begin{align}
\frac{1}{n \overline{d_0}} \frac{K}{\delta}
(1-\varepsilon_1) 
& \le
S(x)
 \le
\frac{1}{n \underline{d_0}} 
\frac{K}{\delta}
(1 + 2 \varepsilon_1),
\quad
\forall x\in \cC.
\label{eq:init-sepa-SC}
\\
 S(x)
& \le
\frac{1}{n \underline{d_0}} \frac{ (1+\varepsilon_2) ( |I| - K) }{1-\delta},
\quad
\forall x \in \cB,
\label{eq:init-sepa-SB} 
\end{align}
If furthermore, Assumption \ref{assump:A2} holds, then

(1) The initial gap between $\cB$ and $\cC$ is at least 
\begin{equation}\label{eq:def-g0-pound}
g_0 
:= 
\frac{1}{n \underline{d_0}} \frac{K}{\delta} (\#),
~
(\#) := \frac{\underline{d_0} (1-\varepsilon_1) }{ \overline{d_0}} 
	- \frac{\delta}{1-\delta} \frac{ |I| - K}{ K } (1+\varepsilon_2),
\end{equation}
that is, $\forall x \in \cC$ and $y \in \cB$, 
$S(x) - S(y) \ge g_0 > 0$.

(2)
At $t=0$, 
\begin{equation}\label{eq:Sx0-upper-bound}
\sup_{x \in \cV} S(x)
\le
\frac{1}{n \underline{d_0}} \frac{K}{\delta} (1+2\varepsilon_1).
\end{equation}
\end{proposition}
Proof in Section \ref{sec:proof}.

\subsection{Stable deformation of $S$ and separation}
We will prove the stability of $S(x,t)$ over time making use of the Hadamard variation formula for the eigenvalues and eigenvectors,
after properly indexing them. 
Specifically, 
since we assume that $\underline{d_0} > 0$, the diagonal matrix $D$ is invertible throughout time,
and the Markov matrix $P=D^{-1}W$ is diagonalizable and similar to $D^{-1/2}WD^{-1/2}$.
Under the matrix perturbation model \eqref{eq:deformW} which is linear in $t$,
the $n$ eigenvalues of the Markov matrix $P$ can be indexed as $\lambda_1(t), \cdots, \lambda_n(t)
$,
so that they are descending at $t=0$, i.e. $\lambda_{k+1}(0) \le \lambda_k(0)$,
 and differentiable with respect to $t$ for $0 \le t \le 1$ (Chapter 2 of \cite{kato2013perturbation}). 
 Similar to the classical Hadamard variation formula,
 the evolution equation of $\lambda_k$ can be shown to be
\begin{equation}\label{eq:lambda-evolve}
\dot{\lambda}_{k} 
 = \psi_{k}^{T}( \dot{W} -\lambda_k \dot{D} )\psi_{k},
\end{equation}
and the equation of the associated eigenvector $\psi_k$ is, when valid,
\begin{equation}\label{eq:psi-evolve}
\dot{\psi}_{k} 
 = 
 -\frac{1}{2} (\psi_{k}^{T}\dot{D}\psi_{k}) \psi_k
 +\sum_{j\neq k}\frac{\psi_{j}^{T}(\dot{W}-\lambda_{k}\dot{D})\psi_{k}}{\lambda_{k}-\lambda_{j}}\psi_{j},
\end{equation}
that is, \eqref{eq:psi-evolve} holds on time intervals when no eigen-crossing of any pair of $\lambda_k$ and $\lambda_j$ happens. 
The derivation of \eqref{eq:lambda-evolve}, \eqref{eq:psi-evolve} is left to Appendix \ref{app:derive}. 

Though the $n$ eigenvalues are ordered from large to small at $t=0$, 
an eigen-crossing (or neighboring eigenvalues becoming very close) 
may happen as $t$ increases,
as illustrated in the diagram in Figure \ref{fig:diag1},
and numerically in the toy example in Figure \ref{fig:toy1}.
This voids a direct adoption of \eqref{eq:psi-evolve} 
unless one shows that the singularity does not affect the differentiability of the 
eigenvector branches before and after the crossing, 
which is still possible in our setting \cite{kato2013perturbation}.
However, even if \eqref{eq:psi-evolve} can be made valid with such an effort,
when an eigen-crossing or a near crossing happens
there is generally no control on the speed of change of the associated pair of eigenvectors.
Some steep changes of eigenvectors are shown in the toy example in Figure \ref{fig:toy1},
at times of (near) eigen-crossings. 
This instability of eigenvectors under matrix perturbation underlies the main difficulty 
to justify the use of leading eigenvectors in this environment,
for both theoretical analysis and algorithms. 

The main observation of this work is to overcome such instability by considering the spectral embedding norm instead of individual eigenvectors.
A key quantity needed in the stability bounds (of both the eigenvalues and the embedding norm) 
is the $\cC$-$\cB$ ``connection strength'', measured by 
\begin{equation}
C:=\sum_{x\in B,\,y\in C}W(x,y).
\label{eq:def-C}
\end{equation}
The analysis needs $C$ to be a small compared to the magnitude of the node degrees, 
specifically, $\frac{C}{\underline{d_0}}$ needs to be a small constant. 
We note that the condition may be much stronger than encountered in applications
due to the reliance on a spectral gap between $I$ and  $I^c$ eigenvalues.
To be specific, we define the $I$-eigen-gap (depending on time $t$) to be 
\begin{equation}
\Delta(t):=\min_{i\in I,\,j\notin I}|\lambda_{i}(t)-\lambda_{j}(t)|,\quad t\in[0,1].
\label{eq:def-Deltat}
\end{equation}
Such an ``$I$-eigen-gap'' prevents eigenvalues from $I$ and $I^c$ to get too close,
but allows arbitrary eigen-crossings within $I$ and within $I^c$.
While needed in the perturbation analysis,
we note that $\Delta(t)$ should be viewed as an artifact due to the limitation of our theory 
(see remark after Theorem \ref{thm:S-sepa}).
However, this is  essentially different from the traditional spectral gap assumed after the $K$-th eigenvalue. 
All proofs in this section are in Section \ref{sec:proof}.

The following proposition proves the preserved $I$-eigen-gap assuming an initial one,
based upon the stable evolution of eigenvalues c.f. \eqref{eq:lambda-evolve}. 
\begin{proposition}[Preservation of $I$-eigen-gap]
\label{prop:I-gap-preserve}
Under (A1), $C$ as in \eqref{eq:def-C}, if for some constant $\Delta>0$, 
$\Delta(0)\ge2\Delta$ and 
\begin{equation}
\frac{C}{ \underline{d_0} }
\le
\frac{1}{8}\Delta,
\label{eq:A4-cond-C}
\end{equation}
then 
\[
\Delta(t)\ge\Delta,\quad\forall0\le t\le1.
\]
\end{proposition}

The significance of the preserved $I$-eigen-gap is that 
we can derive the evolution equation of the embedding norm $S(x,t)$ without being concerned with the eigen-crossings within $I$ (and within $I^c$). 
This is possible by relying on $S(x)$ being the $(x,x)$-th diagonal entry of the spectral projection matrix $P_I :=\sum_{k \in I} \psi_k \psi_k^T$,
which can be written in form of a contour integral of the resolvent in the complex plane where the contour 
circles the eigenvalues in $I$ throughout $t \in [0,1]$, as illustrated in Figure \ref{fig:diag1}. 
The evolution equation below only requires eigenvalue difference $\lambda_{k}-\lambda_{j}$ to be non-vanishing 
when one is from $I$ and the other is from $I^c$.
Actually, this difference is bounded from below by the constant $\Delta$ by Proposition \ref{prop:I-gap-preserve}. 
\begin{proposition}[Evolution of $S$]
\label{prop:S-evolve}
When Proposition \ref{prop:I-gap-preserve} applies, for $0 < t < 1$,
\begin{equation}
\label{eq:S-evolve}
\begin{split}
\frac{\partial}{\partial t} S(x,t)
&  = 
 -\sum_{k\in I,\, j \in I} (\psi_{j}^{T}\dot{D}\psi_{k}) \psi_k(x)\psi_j(x)  + 2 
 \sum_{k\in I, \, j \notin I}
 \frac{\psi_{j}^{T} ( \dot{W} -\lambda_k \dot{D} )\psi_{k}}{\lambda_{k}-\lambda_{j}} \psi_k(x) \psi_{j}(x).
 \end{split}
\end{equation}
\end{proposition}
We then derive the main result: 

\begin{theorem}[Separation at $t=1$]
\label{thm:S-sepa}
Under (A1)-(A3), if for some constant $\Delta>0$, the following conditions are satisfied:

(i) $\Delta(0)\ge2\Delta$, $\Delta(t)$ as in \eqref{eq:def-Deltat},

(ii) $C$ as in \eqref{eq:def-C},
\begin{equation}\label{eq:cond2-thm}
\frac{ C }{ \underline{d_0} } 
\le
\frac{\Delta}{8}
\frac{1}{1+\frac{\Delta}{4}} 
\log \left( 
	1 + \frac{1}{2} \cdot 
 	\frac{ (\#) }{ 1+ 2 \varepsilon_1 } 
	\right)
\end{equation}
where $(\#)$ is defined in \eqref{eq:def-g0-pound},
 and $(\#) > 0$ under Assumption \ref{assump:A2}.

Then the two parts $\cB$ and $\cC$ can be separated by thresholding
the embedding norm, i.e., there exists a constant $\tau$ s.t. at $t=1$
\begin{align*}
S(x) & > \tau, \quad\forall x\in \cC,\\
S(x) & < \tau, \quad\forall x\in \cB.
\end{align*}
\end{theorem}

In practice, $\tau$ can be set to certain quantile of the empirical values of $S(x)$ on all the nodes.

We now make a few comments on the assumptions needed in Theorem \ref{thm:S-sepa}.
Firstly, the r.h.s. of~\eqref{eq:cond2-thm} is technical,
and we show that in the typical setting it is not more restrictive than \eqref{eq:A4-cond-C} 
(and it implies the latter, shown in the proof): 
Note that the r.h.s. is greater than (using that $\log(1+\frac{x}{2}) > \frac{2}{5}x$ for $0 < x \le 1$)
\begin{equation}\label{eq:cond-C-ii-1}
\frac{\Delta}{8} \cdot \frac{1}{1+\frac{\Delta}{4}}  \frac{2}{5} 
\cdot
\frac{ (\#) }{ 1+ 2 \varepsilon_1 },
\end{equation}
and thus unless $(\#)$
is too small, this term would be comparable to $\frac{\Delta}{8}$.
To be specific,
suppose that $\delta$ is so small that 
the first term in the formula of $(\#)$ \eqref{eq:def-g0-pound} dominates,
which makes $(\#)$ approximately
$\frac{ \underline{d_0} }{ \overline{d_0} }$, 
assuming that $\varepsilon_1$ and $\varepsilon_2$ are small constants.
This is reasonable since we typically apply the proposed method when the initial separation is large,
where the initial gap $g_0 = \frac{K}{n \underline{d_0}} \frac{ (\#) }{\delta}$.
Furthermore, in such cases,
if the graph has balanced degree,
 i.e.,
 $ \underline{d_0} \approx \overline{d_0} $,
 then 
$(\#)$ would be close to 1. 
Combined with $\Delta$ being small, 
e.g., $\Delta < 0.1$,
\eqref{eq:cond-C-ii-1} is then approximately $0.39 \cdot \frac{\Delta}{8}$. 

Furthermore, we note that, in practice, 
the embedding norm $S$ can successfully separate $\cC$ from $\cB$
even when the spectral gap requirement and the condition on small $C$ required by the theorem are not satisfied,
see e.g. Figure \ref{fig:toy2-exp}. 
This suggests that the analysis here is likely to be not tight:
for one thing, 
the relaxation of the term $\frac{1}{(\lambda_k -\lambda_j)}$ for $k \in I$, $j\in I^c$ by $\frac{1}{\Delta}$ is crude,
and can be improved, e.g., under proper assumptions of the eigenvalue distribution.
We conjecture that the requirement on $C$ in \eqref{eq:cond2-thm} and the need for an $I$-eigen gap 
 are more restrictive than what occurs in practical applications,
 and further analysis should be able to relax these constraints.

\subsection{Extensions of the analysis}
The main result Theorem \ref{thm:S-sepa} extends to the following cases,
with proof sketches given.

1. 
{\it Weighed embedding norm}.
The definition of the embedding norm $S$ can be generalized as 
\begin{equation}\label{eq:def-S-weighted}
S(x) = \sum_{k \in I} f(\lambda_k) \psi_k(x)^2,
\end{equation}
where $f(\lambda)$ is a (complex) analytic function which is real-valued on real $\lambda$. 
With $f$ being a power of $\lambda$ and certain exponential function,
$S$ is related to  {\it diffusion distance}  \cite{coifman2006diffusion} 
and heat kernel signature \cite{sun2009concise}  respectively,
to be discussed more in the last section. 

We have been addressing the special case where $f=1$.
To extend the analysis to any analytic $f$,
consider the contour integral of $f(z)R(z)$, $R$ being the resolvent (defined in \eqref{eq:formula-R-2}), 
and then the time-evolution equation of $S(x,t)$ can be shown to be
\begin{align*}
 \dot{S}(x) 
& =  
2 \sum_{\substack{k \in I\\ j \notin I}} \frac{ f(\lambda_k) }{\lambda_k - \lambda_j}( \psi_k^T ( \dot{W} -\lambda_k \dot{D}  ) \psi_j ) \psi_k(x) \psi_j(x) \\
&+ \sum_{\substack{k, j \in I\\ k\neq j} }
		\left( 
		\frac{ f(\lambda_k) - f(\lambda_j)}{\lambda_k - \lambda_j} \psi_k^T \dot{W} \psi_j  -  \frac{ \lambda_k f(\lambda_k) - \lambda_j f(\lambda_j)}{\lambda_k - \lambda_j} \psi_k^T \dot{D} \psi_j
		 \right)
		\psi_k(x) \psi_j(x) \\
& + \sum_{k \in I} ( f'(\lambda_k) \psi_k^T \dot{W} \psi_k - (zf(z))' (\lambda_k) \psi_k^T \dot{D} \psi_k ) \psi_k(x)^2,
\end{align*}
where in case that the eigenvalues $\lambda_k$ and $\lambda_j$ coincide, the term $\frac{ f(\lambda_k) - f(\lambda_j)}{\lambda_k - \lambda_j}$
is replaced by $f'(\lambda_k)$ 
and $\frac{ \lambda_k f(\lambda_k) - \lambda_j f(\lambda_j)}{\lambda_k - \lambda_j} $ by
$ (zf(z))' (\lambda_k)$. 
So the r.h.s. is well-defined when an eigen-crossing within $I$ happens,
and the terms $\frac{ f(\lambda_k) - f(\lambda_j)}{\lambda_k - \lambda_j}$ and $\frac{ \lambda_k f(\lambda_k) - \lambda_j f(\lambda_j)}{\lambda_k - \lambda_j} $
are uniformly bounded due to the analyticity of $f$. 
When $f=1$, the equation reduces to \eqref{eq:S-evolve}. 
Proceeding with the same technique as in the proof of the main result, 
the deformation bound of $S(x,t)$ will then involve constant factors which depend on the boundedness of $f$ and $f'$ on $[0,1]$.
Specifically, the constant $\tilde{C}$ will need to be redefined to be 
$(c_1 + \frac{c_2}{\Delta})\frac{2C}{\underline{d_0}}$ where $c_1$ and $c_2$ are absolute constants. 
E.g., when $f(\lambda) = \lambda^p$, $p > 0$, $c_2$ remains 4 (which is the dominating term with small $\Delta$) and $c_1 = (p+1)$.

2. {\it Unequal cluster size in $\cC$}.
The requirement of equal cluster size of the $K$ clusters in $\cC$ can be relaxed. 
Specifically, suppose that the $K$ clusters have varying sizes $|C_j| = \delta_j n$, and $\sum_{k=1}^K \delta_j = \delta$. 
Let $\delta_{\text{min}} = \underset{1 \le j \le K} {\min} \delta_j$, and similarly define $\delta_{\text{max}}$.
Then under Assumption \ref{assump:A1}, 
\eqref{eq:init-sepa-SC} and \eqref{eq:init-sepa-SB} become 
\begin{align}
\frac{ 1-\varepsilon_1 }{n \overline{d_0} \delta_j } 
& \le
S(x,0)
 \le
\frac{1}{n \underline{d_0}} 
(
\frac{\varepsilon_1}{\delta/K} + \frac{1 + \varepsilon_1}{ \delta_j }
),
\quad
\forall x\in C_j \text{ for some $j$}.
\label{eq:init-sepa-SC-2}
\\
 S(x,0)
& \le
\frac{1}{n \underline{d_0}} \frac{ (1+\varepsilon_2) ( |I| - K) }{1-\delta},
\quad
\forall x \in \cB,
\label{eq:init-sepa-SB-2} 
\end{align}
Define  for $j=1,\cdots, K$,
\begin{equation}\label{eq:def-gj}
g_{j,0} :=
 \frac{1}{n \underline{d_0}} 
\left(  \frac{\underline{d_0} (1-\varepsilon_1) }{ \overline{d_0} \delta_j} 
 - \frac{ (1+\varepsilon_2) ( |I| - K) }{1-\delta} \right),
\end{equation}
and the minimum of $g_{j,0}$ is 
\[
g_{\text{min},0} =
\frac{1}{n \underline{d_0}} 
\left(  \frac{\underline{d_0} (1-\varepsilon_1) }{ \overline{d_0} \delta_{\text{max}}} 
 - \frac{ (1+\varepsilon_2) ( |I| - K) }{1-\delta} \right).
\]
Modify Assumption \ref{assump:A2} to be that $g_{\text{min},0} > 0$,
then the initial separation of $S(x,0)$ on $\cC$ and $\cB$ is at least $g_{\text{min},0}$
(and more precisely $g_{j,0}$ between $C_j$ and $\cB$),
and 
\eqref{eq:Sx0-upper-bound} becomes 
\begin{equation}\label{eq:upper-bound-S0-3}
\bar{S}(0) = 
\sup_{x \in \cV} S(x)
\le
\frac{1}{n \underline{d_0}} ( \frac{\varepsilon_1 }{\delta / K} + \frac{1+\varepsilon_1}{\delta_{\text{min}}} ).
\end{equation}
Note that Proposition \ref{prop:I-gap-preserve}, Proposition \ref{prop:S-evolve} and claims (1) (2) in the proof of Theorem \ref{thm:S-sepa} 
do not rely on Assumption \ref{assump:A1} or Assumption \ref{assump:A2} and are valid. 
As a result, it can be shown that the $t=1$ separation between $\cC$ and $\cB$ by $S$ holds as long as 
\begin{equation}\label{eq:g-condition-3}
g_{\text{min},0} \ge 
2 (e^{\tilde{C}} -1) \frac{1}{n \underline{d_0}} ( \frac{\varepsilon_1 }{\delta / K} + \frac{1+\varepsilon_1}{\delta_{\text{min}}} ),
\quad 
\tilde{C} = \left( 1+\frac{4}{\Delta} \right) \frac{2 C }{ \underline{d_0}}.
\end{equation}
This condition is more restrictive when the cluster sizes in $\cC$ are less balanced, 
namely when the difference 
$\delta_{\text{max}} - \delta_{\text{min}}$ becomes larger.
In our numerical experiments, all the sub-clusters are of comparable sizes 
(in the outlier detection in images, $K=1$ or 2, and in image segmentation the clusters are of similar sizes),
while we note that an 
extremely unbalanced cluster size, e.g., very small $\delta_{\text{min}}$, could affect the performance of the method.

3. {\it Detection of parts of $\cC$}.
The above argument leads to a ``personalized'' detection condition for each cluster $C_j$ in $\cC$,
that is, even when $S(x,1)$ fails to separate some clusters in $\cC$ from $\cB$ it may still successfully detect the rest. 
To see this, note that the proof of the theorem actually gives the following: 
For any subsets $E_1$ and $E_2$ of $\cV$, if
\[
\min_{x \in E_1} S(x,0) - \max_{x \in E_2} S(x,0) \ge g_{E_1, E_2} > 0
\]
then $E_1$ and $E_2$ can be separated by $S(x,1)$ by some threshold as long as \[
g_{E_1, E_2} \ge 2 (e^{\tilde{C}} -1) \bar{S}(0).
\]
The previous results corresponds to $E_1 = \cC$ and $E_2 = \cB$.
Let $E_1$ be any individual cluster $C_j$,
then since $\bar{S}(0)$ is upper-bounded by \eqref{eq:upper-bound-S0-3}, 
we have that each cluster $C_j$ can be separated from $\cB$ by $S(x,1)$ if $g_{j,0}$ as defined in \eqref{eq:def-gj} 
is larger than the r.h.s. of \eqref{eq:g-condition-3}. 

4. {\it Initial inclusion of $I$}. 
The Assumption \ref{assump:A1} (a) can be relaxed by only requiring $K'$ $\cC$-eigenvectors in $I$, $0< K' \le K$,
as long as they contribute to a sufficiently large $S(x,0)$ on $\cC$, 
or any subset of $\cC$ such as an individual cluster $C_j$.
The separation guarantee at time 1 follows the same argument as in item 2. above,
where the quantities \eqref{eq:init-sepa-SC-2} \eqref{eq:init-sepa-SB-2} and consequently \eqref{eq:def-gj} \eqref{eq:upper-bound-S0-3}
need to be modified. The precise condition is not pursued here. 
In practice, this means that even if less than $K$ ``nearly'' $\cC$-eigenvectors are included in $I$, the method may still be able to detect part of $\cC$ from $\cB$.

%
%
\section{Experiments}\label{sec:4}
In this section we will apply the spectral embedding norm to both synthetic and real-world datasets, in scenarios of both single outliers and multiple clusters in a cluttered background. 
Codes are available at \url{https://github.com/xycheng/EmbeddingNorm}.

\begin{figure}[t]
\centering{
 \subfloat[]{ \includegraphics[height=0.2\linewidth]{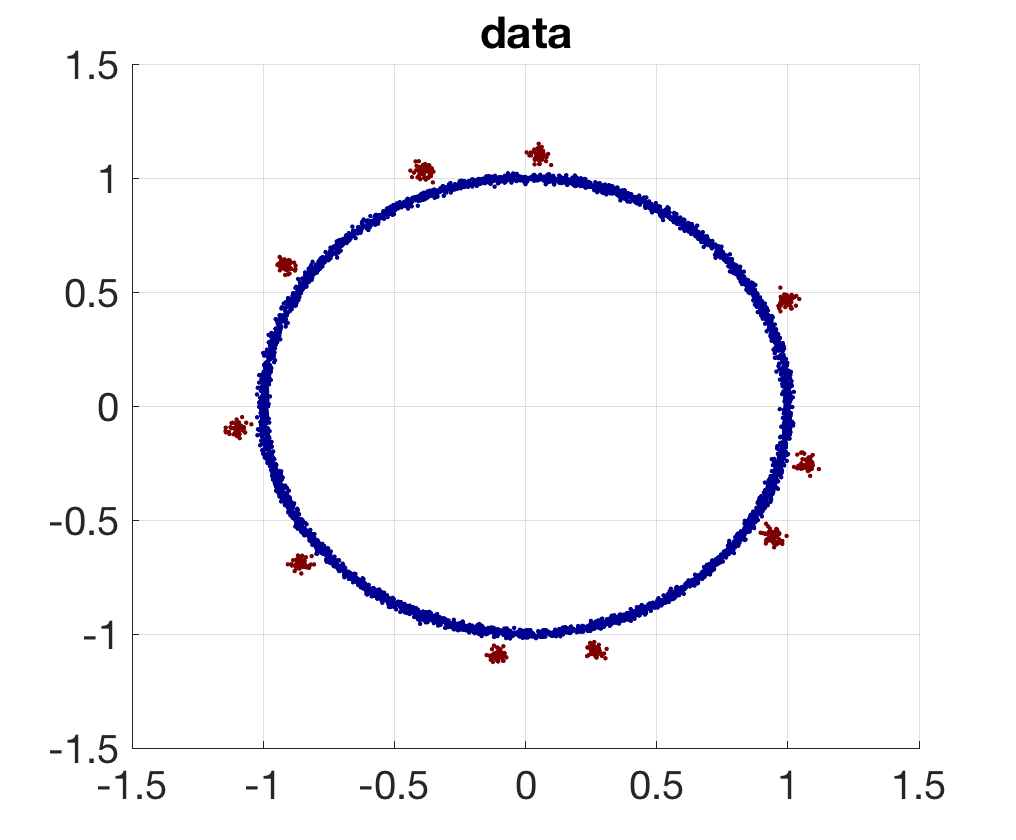} \hspace{-10pt} } 
 \subfloat[]{ \includegraphics[height=0.2\linewidth]{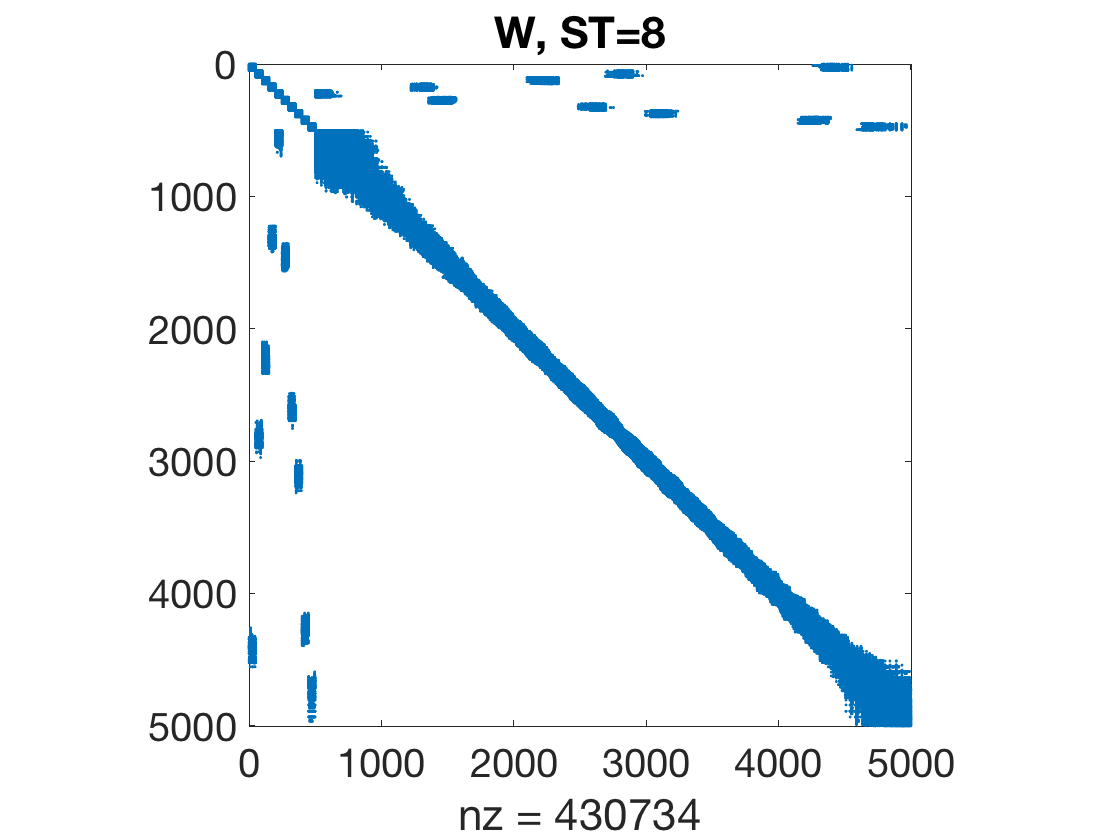} \hspace{-10pt}}
 \subfloat[]{ \includegraphics[height=0.2\linewidth]{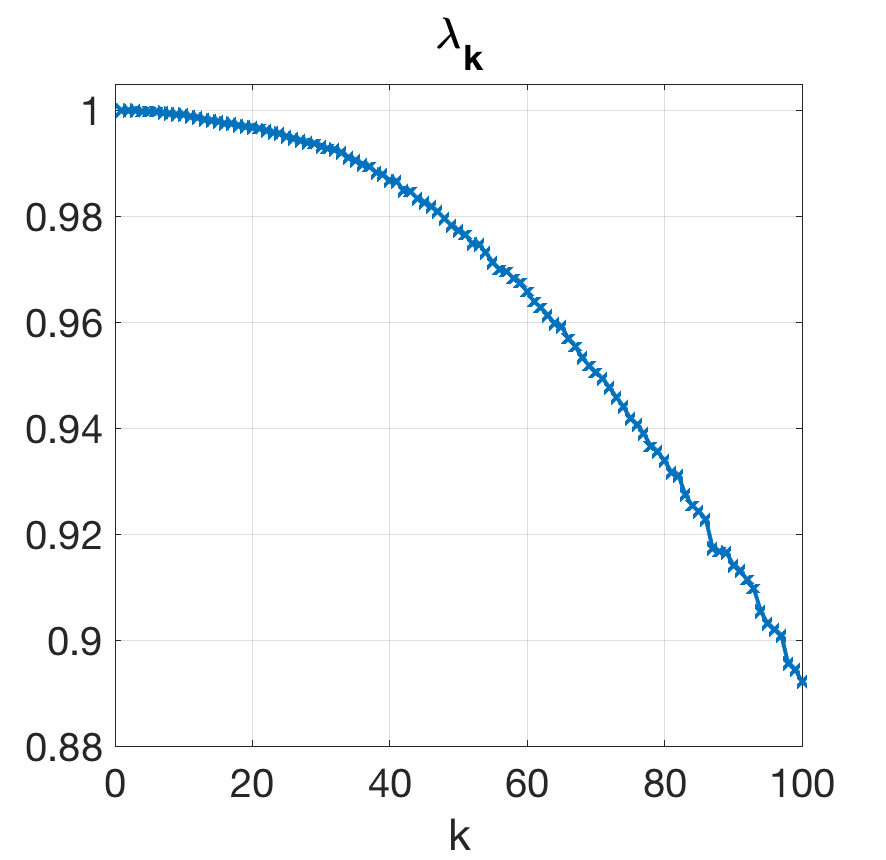}  \hspace{-5pt}}
 \subfloat[]{ \includegraphics[height=0.2\linewidth]{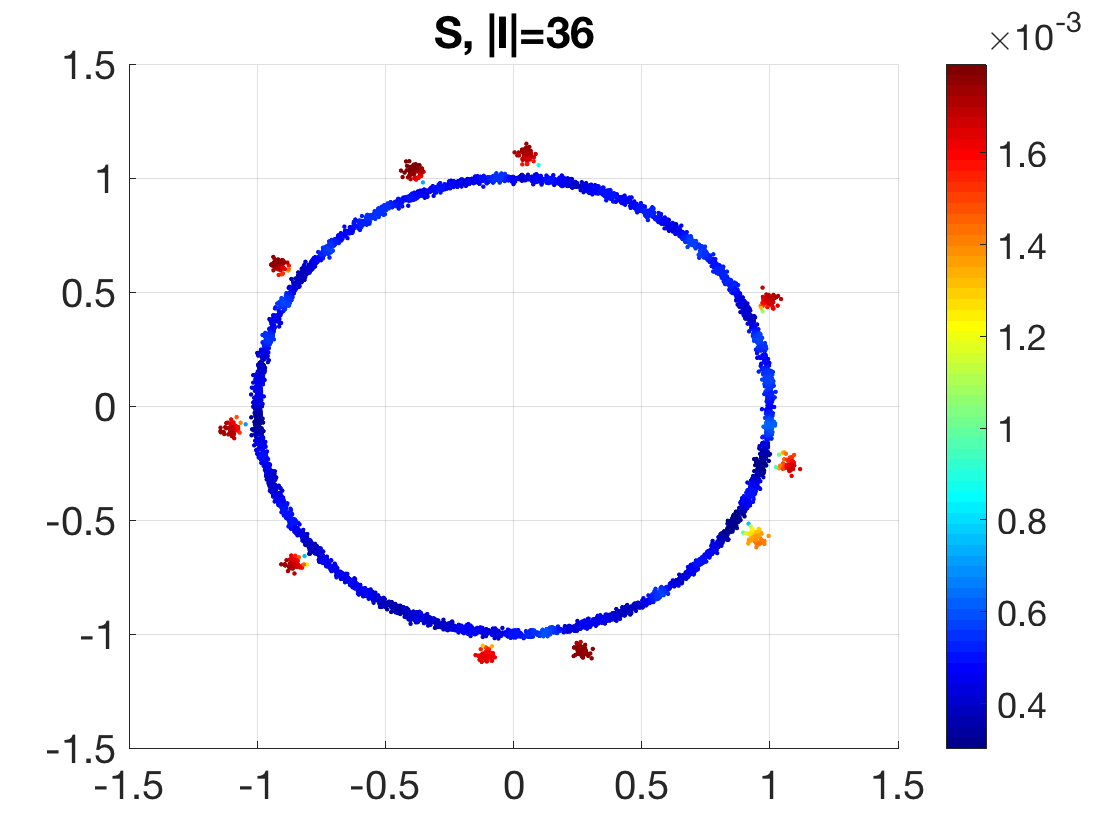} }\\
 \subfloat[]{\hspace{-15pt} 
 	\includegraphics[height=0.295 \linewidth]{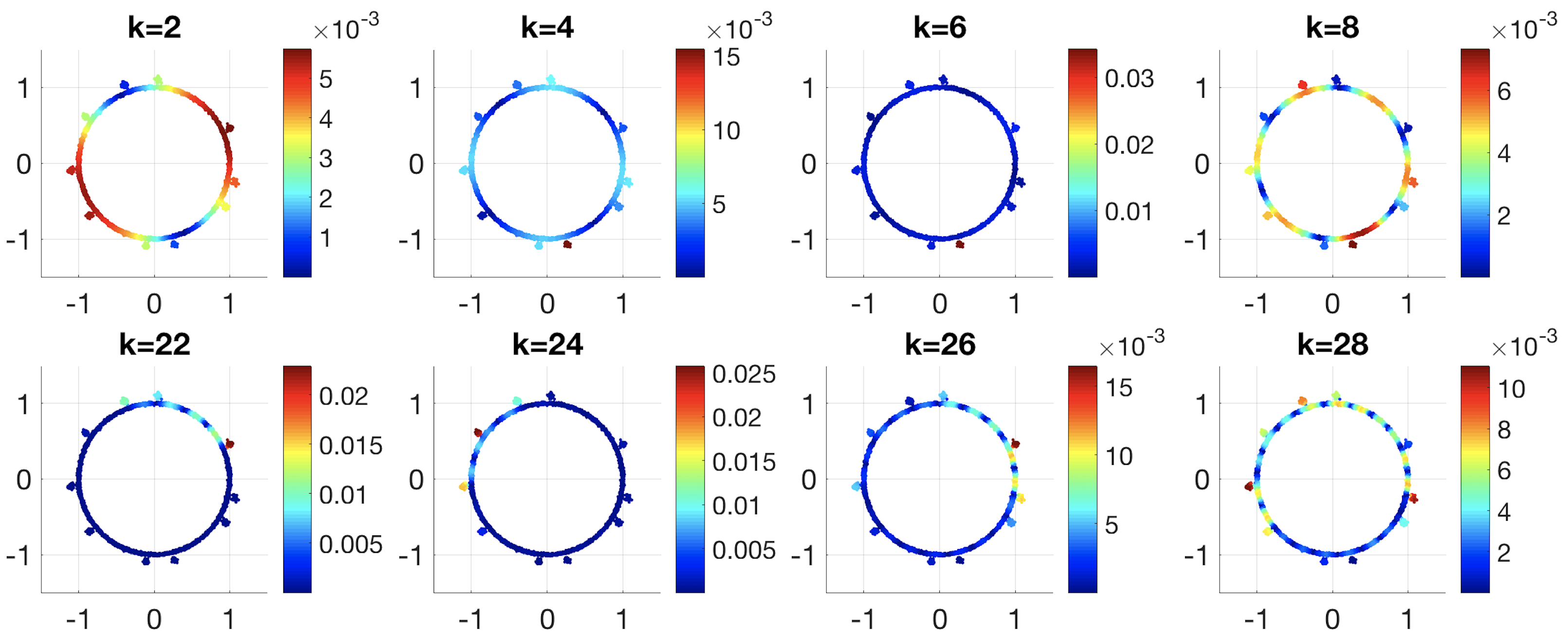} 
 \hspace{-8pt}}
 \subfloat[]{ \includegraphics[height=0.298 \linewidth]{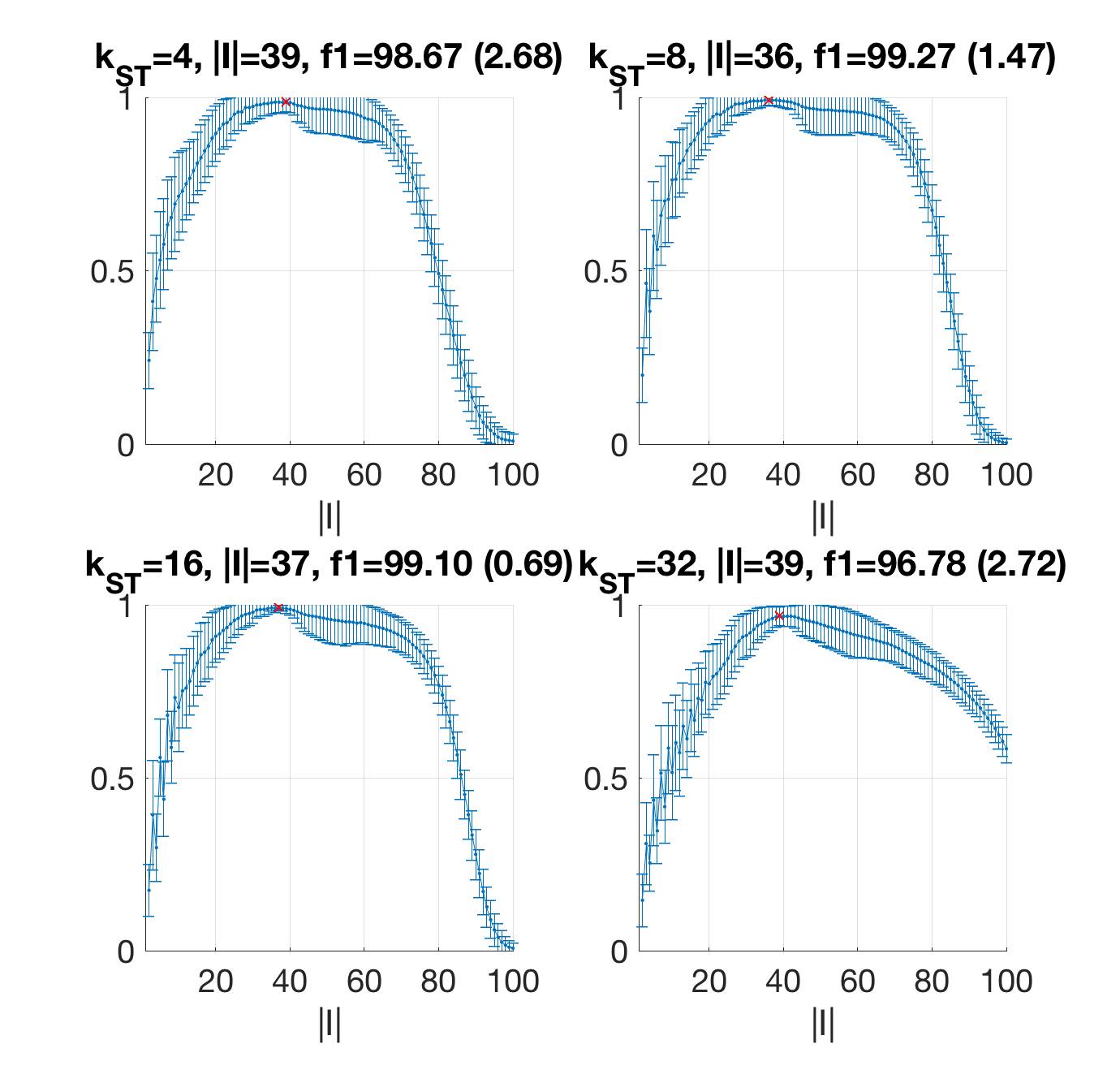}}
}
\vspace{-10pt}
\caption{ 
Detection of $\cC$ from a manifold-like $\cB$:
(A) $n=5000$ data points in $\R^2$ sampled on ${\cB} \cup {\cC}$, 
where ${\cB}$ points lie close to a circle (blue) and
 ${\cC}$ points form $K=10$ clusters lying nearby (red), $\delta = 0.1$.
(B) The affinity matrix $W(t)$ at $t=1$, c.f. \eqref{eq:deformW}.
(C) The first 100 eigenvalues of the Markov matrix.
(D) The plot of $S$.
(E) $k$-th Eigenvectors of multiple $k$'s of the Markov matrix. 
(F) F1-score of the detection of $\cC$ by thresholding the values of $S_I$, where $|I|$ varies from 2 to 100, 
and for multiple choice of self-tuning parameter ($k$-nearest neighbor in self-tuning, denoted by $k_{\text{ST}}$). 
Mean and standard deviation of F1-score are shown, and optimal value of $|I|$ are indicated by a red cross. 
}
\label{fig:toy2-exp}
\end{figure}

\begin{figure}[t]
\centering{
\subfloat[]{\hspace{-15pt}   \includegraphics[trim=20 20 20 5, clip,height=0.19\linewidth]{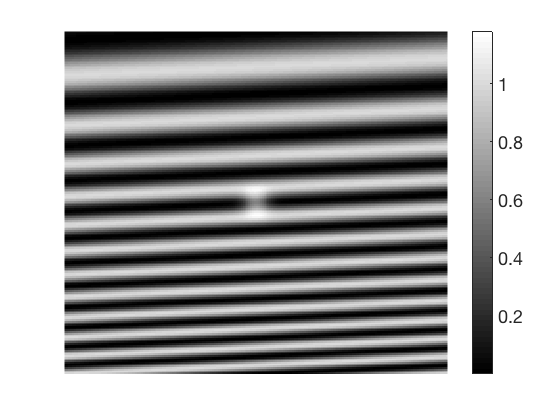} \hspace{-5pt}}
\subfloat[] {\hspace{-5pt}  \includegraphics[trim=20 20 20 5, clip,height=0.19\linewidth]{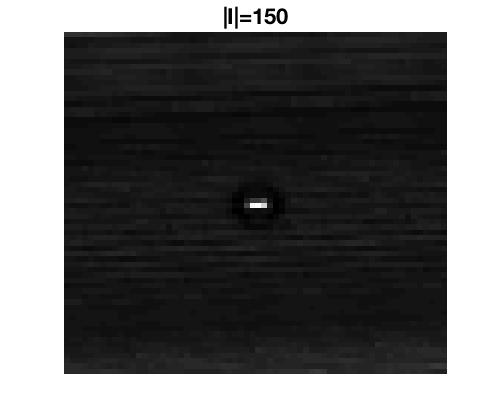}\hspace{-5pt}}
\subfloat[] {\hspace{-5pt} \includegraphics[trim=20 20 20 5, clip,height=0.19\linewidth]{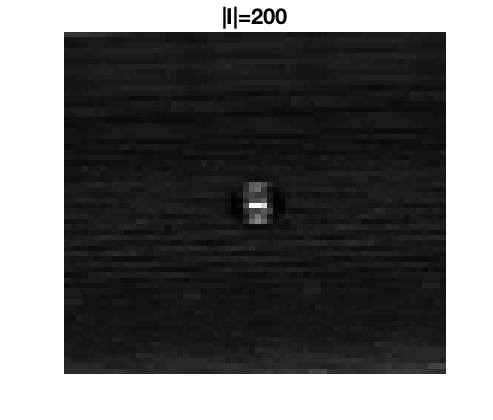}\hspace{-5pt}}
\subfloat[] {\hspace{-5pt} \includegraphics[trim=20 20 20 5, clip,height=0.19\linewidth]{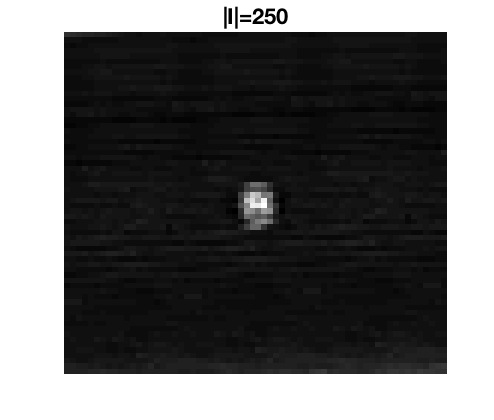}\hspace{-5pt}}
\subfloat[] {\hspace{-5pt} \includegraphics[trim=20 20 20 5, clip,height=0.19\linewidth]{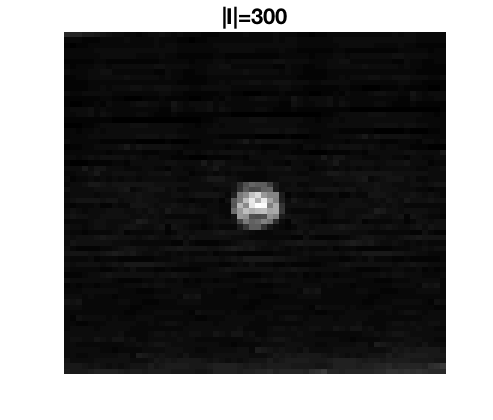}}\\
\subfloat[]{ \hspace{-15pt} \includegraphics[trim=20 20 20 5,height=0.27\linewidth]{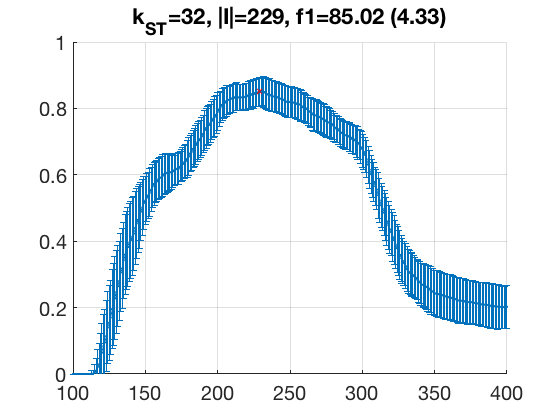} \hspace{-5pt}}
\subfloat[]{\includegraphics[trim=120 30 100 20, clip, height=0.28\linewidth]{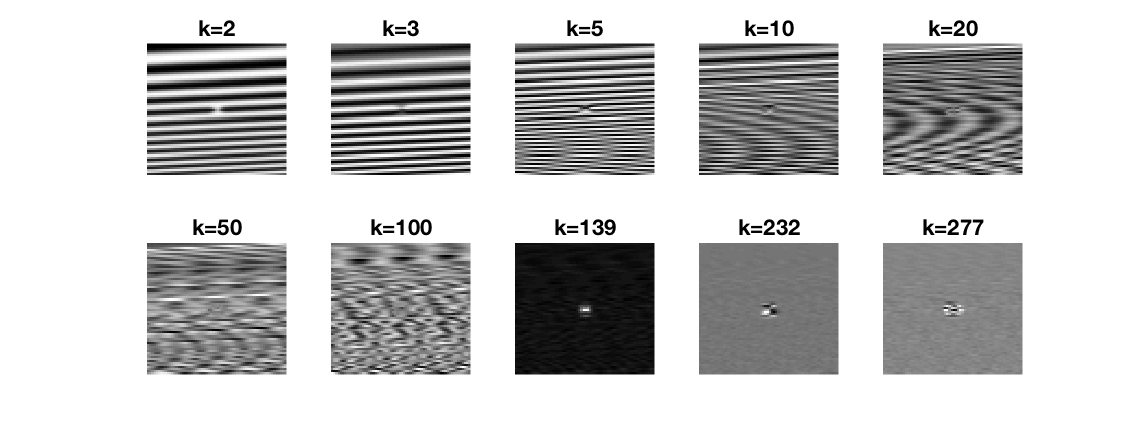}}
}
\vspace{-10pt}
\caption{ 
Detection of outlier patches in a synthetic image. 
(A) Image consisting of background stripes and an anomaly region in the center, $\delta \approx 0.01$.  
(B-E)  Embedding norm $S_I $ computed with $|I|=150$, 200, 250, 300 respectively, 
plotted as images where the pixel value indicates the value of $S_I$ and the pixel location is the center of every image patch. 
(F) F1-score of the detection of outlier patches 
by thresholding the values of $S_I$, 
where $|I|$ varies from 100 to 400.
Randomness is attained by subsampling image patches and running independent replicas of the experiment.
Mean and standard deviation of F1-score are shown, and optimal value of $|I|$ are indicated by a red cross. 
(G) Eigenvectors for various index $k$, plotted as images.
The last three show high-index exemplar eigenvectors which are localized on the outlier patches,
while the first 100 eigenvectors are mainly supported on the background. 
}
\label{fig:patchoutlier-exp}
\end{figure}

\subsection{Manifold data toy example}\label{subsec:4.1}
We begin with a simulated dataset in $\R^2$ composed of a manifold-like background $\cB$ and clusters in $\cC$ according to the following model. The background $\cB$ consists of i.i.d samples $x_i $ distributed as $x_i = y_i +n_i$,
where $y_i$ are uniformly distributed on the unit circle, which is a one-dimensional manifold,
and $n_i \sim \mathcal{N} (0, \epsilon_{\cB}^2 I)$, with $\epsilon_{\cB} = 0.01$.
$\cC$ contain $K$ equal-sized sub-clusters, each has i.i.d. samples drawn from $\mathcal{N} (\mu_j, \epsilon_{\cC}^2 I)$,
where $\mu_j$ are centered close to the circle, and $\epsilon_{\cC} = 0.02$. 
We generate $n=5000$ points,
and the number of points in $\cC$ is set to be $\delta n$ for positive $\delta$, rounded to the closest integer.
To measure the accuracy of the detection of $\cC$ we compute the F1 score:
$\textrm{F1} = \frac{2pr}{p+r}$, 
where $p:=\frac{ \text{TP}}{ \text{TP} + \text{FP}}$, $r := \frac{ \text{TP}}{\text{TP} + \text{FN}}$,
and TP, FP and FN stand for True Positive, False Positive and False Negative respectively. 
The $\delta$-quantile of the empirical values of $S$ is used as the threshold $\tau$ when computing the classification.

Figure \ref{fig:toy2-exp} shows results for a typical realization of the dataset with $K=10$, $\delta = 0.1$.
From (C) it can be seen that  the eigenvalues do not reveal any clear eigen-gap at $K=10$.
The first $K$ eigenvectors do not give a clear indication of the cluster $\cC$, but are mainly supported on the $\cB$-eigenvectors,
as shown in (E) for $k=2$ and 8.
Examining up to the first $k \sim 40$ ones, 
certain eigenvectors are more localized on $\cC$ when $k > K$, 
e.g., $k=24$ and 26.
The embedding norm $S_I$ clearly separates $\cC$ from $\cB$, as shown in (D).
The results are not sensitive to algorithmic parameter choices. 
Let $k_{\text{ST}}$ be the $k$-nearest neighbor used to set the local self-tuning scale~\cite{selftune} in constructing the affinity matrix $W$. 
Then, throughout varying values of the parameter $k_{\text{ST}}$,
the F1-score of the detection by thresholding $S$ reveals a ``plateau'' of valid values of $|I|$, 
e.g., when $k_{\text{ST}} = 8$, the range of $|I|$ is about $30 \sim 45$,
with the optimal F1 score obtained at $|I|=$36. 
The best F1 score for $k_{\text{ST}} = 4$, 8 or 16 are all greater than 0.98. 

Similar results are obtained for smaller $K=2$ where $\delta = 0.02$,
as shown in Figure \ref{fig:toy1-exp}. 
The condition  \eqref{eq:assumpA2} 
in Assumption \ref{assump:A2}
suggests that $|I|$ is chosen to be proportional to $\frac{K}{\delta}$,
and this is revealed in Figure \ref{fig:toy2-exp} and Figure \ref{fig:toy1-exp} (in these two examples $\frac{K}{\delta}$ is kept to be the same)
as the plateaus of valid $|I|$ are at about the same range, across values of $k_{\text{ST}}$.

\begin{figure}[t]
\centering{\includegraphics[width=1.0\linewidth]{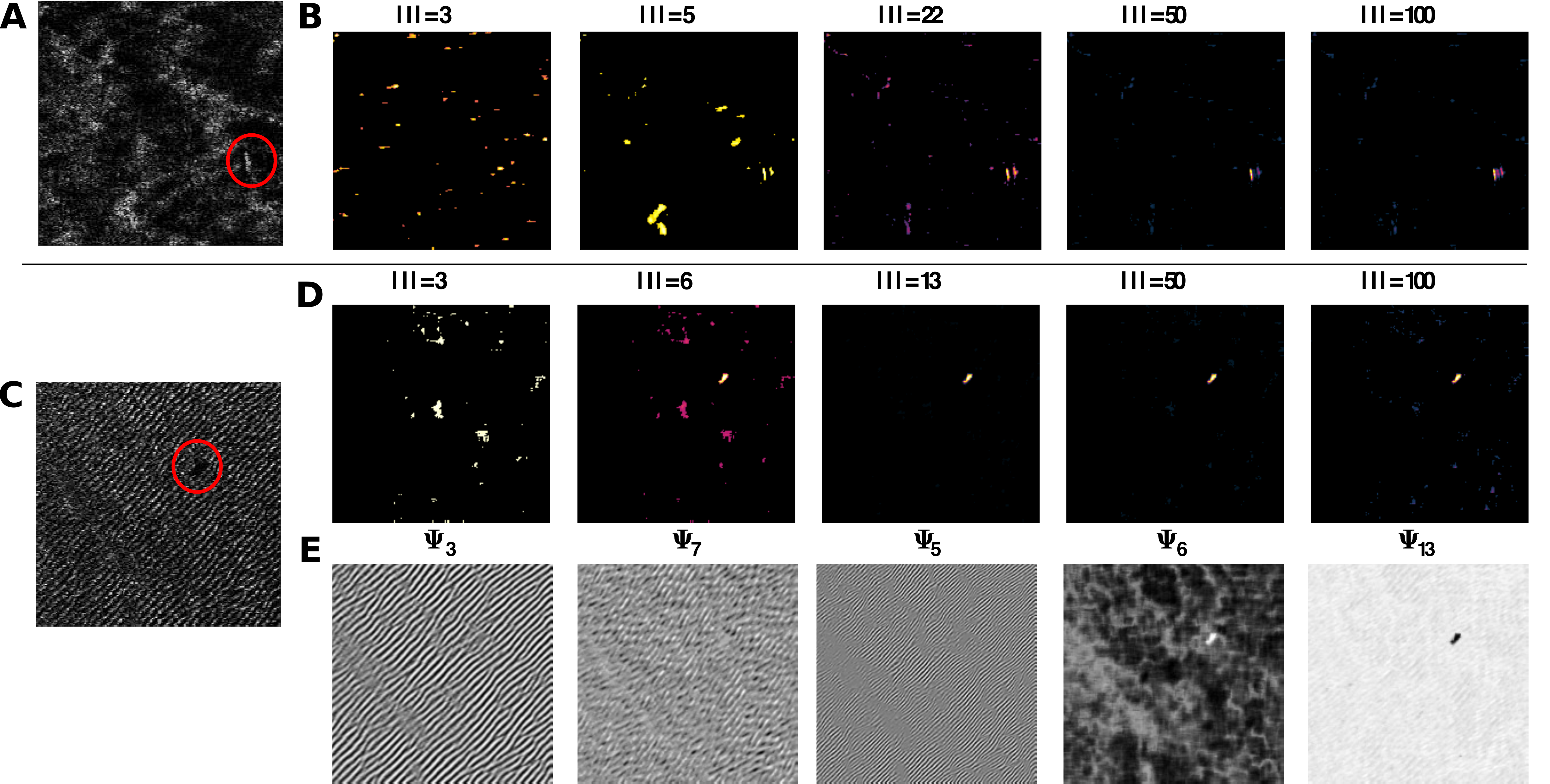}}
\caption{Detecting anomalies in side-scan sonar images. 
A. left: Side-scan sonar images with sea-mines indicated by red circle. right: Spectral embedding norm with $|I|=50$ reveals the sea-mine.
B. Examples of Laplacian eigenvectors with eigenvectors mainly supported on the background (top) or localizing on the sea-mine (bottom). 
C. Spectral embedding norm $S_I$ for top side-scan sonar image with increasing $|I|$. The sea-mine first ``appears'' for $|I|=6$. For $|I|=13$ the sea-mine separates cleanly from the background, however as $|I|$ increases to 100 spectral norm starts to reveal background components as well. 
D. The target sea-mine first ``appears'' for $|I|=5$, however it then recedes into the background (not displayed). For $|I|=100$, the sea-mine is well separated from the background.
E. Exemplar eigenvectors plotted as an image demonstrating localization on the background in the three left plots and localization on the sea-mine in the two right plots.}
\label{fig:sea-mines}
\end{figure}

\subsection{Image anomaly detection}\label{sec:4-2}
Anomaly detection can be seen as a special case of
clustering in which there is a vast imbalance in the size of clusters,
i.e., background vs. anomaly, and the density of each cluster. 
In image anomaly detection, the goal is to detection a small compact region (subset of connected pixels) that differs from the normal image background.
In general, we can assume the number of anomalies $K$ to be small or even 1.

Figure \ref{fig:patchoutlier-exp} shows the numerical result on a synthetic image which consists of anomaly patches against a slowly varying background,
as a model of patterned images.
The image is a function on $[-1,1]^2$ expressed as
\[
I(x,y) = \left( 1+\frac{1}{2} \cos( (0.05 x + y + 1.5)^2 \cdot 2\pi)\right) + 0.6 \exp\{ - \frac{x^2+y^2}{2\cdot 0.05^2}\},
\]
where the first term models the background stripes, 
and the second term models the outlier region in the center, as shown in (A). 
The image size is $200\times 200$, and $n=4096$ image patches of size $9 \times 9$ are extracted with a stride of 3. 
True outlier patches are identified by thresholding the value of the Gaussian bump with a fraction of $\delta \approx 0.01$. 
The graph affinity is computed with self-tuning bandwidth~\cite{selftune} and $k_{\text{ST}} = 32$.
The spectral embedding norm computed for various values of $|I|$ is shown in (B-E), all of which identify the outlier region with improved performance for $|I|=200$ and 250. 
The leading eigenvectors up to $k=100$ fail to be indicative of the outlier region, 
while those with higher $k$ may be, as shown in (G).
To quantitatively evaluate the performance of the proposed method, 
we threshold $S_I$ at the $0.99$ quantile to detect the anomaly patches,
and compute the F1 score. We repeat $100$ experiments by randomly subsampling 3000 patches and the averaged F1 score is shown in (F) with standard deviation. The method achieves an average best $F1 = 85.02$ when $|I|=229$. The results are similar with $k_{\text{ST}} = 16$ and 64.

In Figure~\ref{fig:sea-mines} we demonstrate on real-world images that eigenvectors localizing on the anomaly can be buried deep within the spectrum of the image, and that by calculating the spectral embedding norm we can separate the anomalies from a cluttered background.
We examine two side-scan sonar images containing a single sea-mine, displayed in (A) and (C), where we consider the sea-mine to be an anomaly (indicated by a red circle).
The sea-mine can appear as either either a bright highlight (A) or a only a dark shadow (C), which is due to the object blocking the sonar waves from reaching the seabed. The background is composed of sea-bed reverberations and exhibits great variability in appearance. 
For the side-scan sonar images, in practice $\delta$ is in the range of $5 \times 10^{-3} \sim 5 \times 10^{-4}$.
To construct the affinity matrix $W$, we extract all overlapping patches of size 8, and build a nearest neighbor graph with 64 neighbors, setting $k_{\text{ST}}=32$.  
(B) and (D) display $S(x)$ for all pixels $x$ in the image, for increasing values of $|I|$. For both images the sea-mine is revealed consistently for a wide range of values, while the background is suppressed. Note that in both cases this requires looking deep enough in the spectrum, and summing over the first few eigenvectors brings out background structures.
Finally, (E) displays eigenvectors $\psi_k$ for the side-scan sonar image in (C), where the three left eigenvectors localize on the background $\cB$, revealing its periodic nature at different scales and orientations, while the eigenvectors in the right two plots localize on the sea-mine $\cC$. 

\begin{figure}[t!]
\centering{\includegraphics[width=1.0\linewidth]{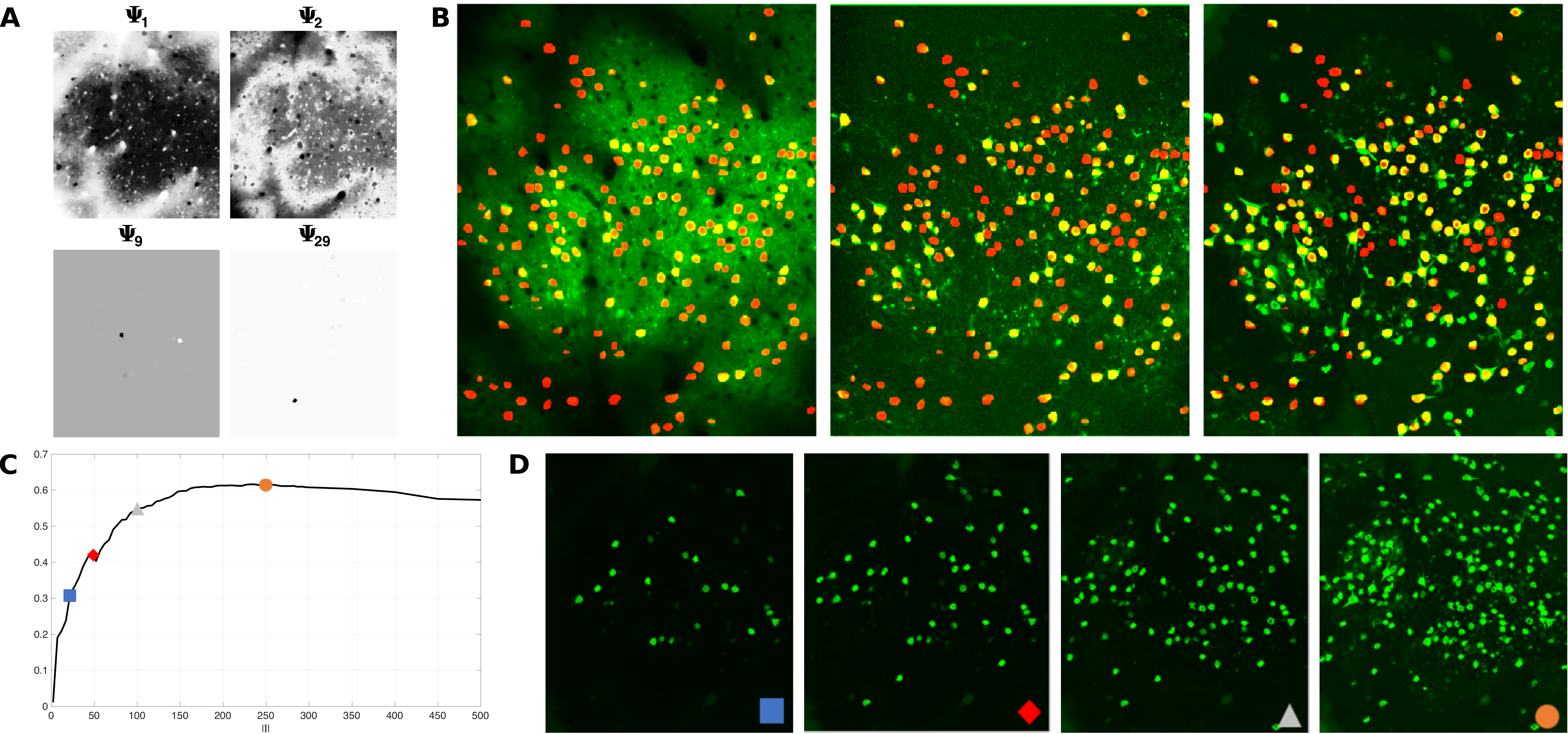}}
\vskip -0.1 in
\caption{
A. Example of Laplacian eigenvectors on a calcium imaging dataset from Neurofinder, 
either mainly supported on the background clutter (top) or localizing on neurons which are clusters (bottom).  B. Images of the Temporal mean (left) Temporal correlation(middle) and Spectral embedding norm (right) for a Neurofinder dataset. The spectral embedding norm has both removed the background (which is present in the mean image) and enhanced the appearance of the structure in the image: neuronal soma and dendrites, with sharp morphology. The correlation image is much noisier with fewer visual neurons. C. F1-score of segmenting neurons from background based on the spectral embedding norm image, for increasing $|I|$ values. For a range of values (200-250) the F1 score plateaus, and then decreases as the number of included eigenvectors increases. D. $S(x)$ for $|I|=20,49,100,250$  demonstrates how more and more clusters are revealed. }
\label{fig:TP_eigvec}
\end{figure}

\subsection{Calcium imaging}\label{sec:4-3}
Calcium imaging is an experimental method in neuroscience that enables imaging the individual activity of hundreds of neurons in an awake behaving animal, at cellular resolution~\cite{RN27}. 
The acquired data is composed of a spatiotemporal volume, where, after motion correction, the neuron locations are fixed and the temporal activity consists of hundreds to tens of thousands of time-frames. 
There is also varying temporal activity in the background (neuropil).
Thus, this data can also be viewed as an image whose pixels lie in a high-dimensional space (time-frames), consisting of hundreds of clusters (neurons) in an image plane with a non-trivial background, which matches our problem setting.

The analysis pipeline of calcium imaging typically includes calculating a 2D image that depicts the structure that exists in this volume and highlights the existing neurons.
Such images serve for manual segmentation, to align imaged volumes across days (where the field of view may shift), to display neurons detected by automatic and manual means, and even for initialization of automatic ROI extraction algorithms methods~\cite{Reynolds2017,Pachitariu2013}.
A common choice is the temporal correlation image~\cite{Smith2010}, or the temporal mean image. 
Here we show that the spectral embedding norm provides a meaningful visualization of the data, with sharp morphology and suppression of noise from the background clutter.  
For these datasets, depending on the brain region and neuron type being imaged, $K$ is in the range of dozens to few hundreds, and $\delta \lessapprox 0.1$.  

In Figure~\ref{fig:TP_eigvec}, we analyze a publicly available dataset from Neurofinder~\cite{berens2017standardizing}. The images are $512\times 512$ pixels and 8000 time frames have been recorded at 8 Hz. Ground truth labels provided with the dataset includes 197 identified neurons, however note that recent papers point out that the ground truth on Neurofinder datasets is probably lacking, i.e. not all neurons are labeled~\cite{Reynolds2017,Spaen2017}.
The affinity matrix $W$ is calculated using a nearest neighbor graph for all pixels, represented as high-dimensional vectors in time, with 50 nearest neighbors. To accelerate the nearest neighbor search dimensionality is reduced from 8000 to 300 using PCA.
(A) displays examples of eigenvectors from both the background (top) and localizing on single neurons (bottom). 

In (B) we compare the spectral embedding norm (right) to the temporal mean (left) and temporal correlation image (middle). 
In each image, the values (mean/correlation/norm) appear in the green channel, while we overlay in the red channel a mask of the ground truth labels that were manually detected (where the two overlap it appears as yellow). 
The mean image exhibits a strong background, while neurons appear as typical ``donuts''~\cite{Pachitariu2013}. In the correlation image, the background mostly appears as noise. In comparison, the background has been suppressed in the spectral embedding norm image, while neurons which are barely or not at all visible in the correlation image appear as bright clusters. 

To quantify, the separation of background and clusters, we segment the spectral embedding norm image for increasing $|I|$, and compare the overlap between the segmented clusters and the given ground-truth mask.
We set the threshold $\tau$ to be the value of the 93-rd percentile of $S$ values for each value of $|I|$. 
In (C) we plot the F1 score for this segmentation and demonstrate a plateau of stable F1-score values for $|I|$ in the range 200--250.
To demonstrate the property of the spectral embedding norm to perform partial detection of $\cC$, we  display $S(x)$ for multiple values of $|I|$ in (D). 
Note that we are not performing clustering here, but rather demonstrating how the embedding norm can be used to separate meaningful structure from background clutter. 
Thus beyond visualization, this approach can then serve to remove the background, and focus only on the remaining clusters in $\cC$, thus simplifying subsequent clustering and data analysis tasks.

\section{Proofs}\label{sec:proof}
%
%
%
%
\begin{proof}[Proof of Proposition \ref{prop:init-sepa}]
It suffices to prove \eqref{eq:init-sepa-SC} and \eqref{eq:init-sepa-SB},
because Assumption \ref{assump:A2} implies that the r.h.s of 
\eqref{eq:init-sepa-SB} is strictly less than
the l.h.s. of \eqref{eq:init-sepa-SC}
by $g_0$, 
and then claims (1) and (2) directly follow.

To prove \eqref{eq:init-sepa-SC}:
Note that for any $x \in \cC$, at $t=0$,
\[
S(x) 
= \sum_{k \in I} \psi_k(x)^2
= \sum_{k \in I \cap \Psi^\cC} \psi_k(x)^2.
\]
By Assumption \ref{assump:A1} (a),
up to a possible $K$-by-$K$ rotation among the $K$ eigenvectors in $I \cap \Psi^\cC$, 
we assume that  $\psi_j$ is the eigenvector associated with the sub-cluster $C_j$,
$j=1,\cdots, K$, and then
\begin{equation}\label{eq:proof-SC-1}
S(x) 
= \sum_{j =1}^K \psi_j(x)^2,
\end{equation}
as the rotation preserves the squared sum. 
Furthermore,
suppose that $x \in C_{j_x}$, Assumption \ref{assump:A1} (b) gives  
\begin{align*}
\psi_j(x)^2 
& \in \frac{1}{\nu(C_j, 0)} [1-\varepsilon_1, 1+\varepsilon_1], \quad j=j_x, \\
\psi_j(x)^2
& \le \frac{\varepsilon_1}{ \nu(C,0)}, \quad j\neq j_x.
\end{align*}
Plugging into \eqref{eq:proof-SC-1}, it shows that
\[
S(x) \le \frac{1+\varepsilon_1}{\nu(C_{j_x}, 0)} + (K-1) \frac{\varepsilon_1}{ \nu(C,0)},
\]
and together with  
$\nu( C_{j}) = \sum_{x\in C_j} d(x,0) \ge \underline{d_0} |C_j| = \underline{d_0} \frac{\delta n}{K}$ for any $j$
(the $K$ sub-clusters are equal-size)
and similarly 
$\nu( C) \ge \underline{d_0} |C| = \underline{d_0} \delta n$,
it gives the upper bound in \eqref{eq:init-sepa-SC}.
Consider the lower bound, \eqref{eq:proof-SC-1} continues as
\[
S(x) 
\ge
\psi_{j_x}(x)^2
\ge 
\frac{1- \varepsilon_1}{\nu(C_{j_x}, 0)}.
\]
Combined with 
$\nu( C_{j}) = \sum_{x\in C_j} d(x,0) \le \overline{d_0} |C_j| = \overline{d_0} \frac{\delta n}{K}$
for any $j$,
this gives the lower bound in \eqref{eq:init-sepa-SC}.

To prove \eqref{eq:init-sepa-SB}:
For any $x \in \cB$, at $t=0$,  
\begin{equation}\label{eq:proof-init-SB-1}
S(x) 
= \sum_{k \in I} \psi_k(x)^2
= \sum_{k \in I \cap \Psi^\cB} \psi_k(x)^2
\le \sum_{k \in I \cap \Psi^\cB} 
	\frac{1+\varepsilon_2}{ \nu(B,0)},
\end{equation}
where
$
\nu(B,0) = \sum_{x \in \cB} d(x,0) 
\ge \underline{d_0} |\cB|
= \underline{d_0} n (1-\delta)$,
and the last inequality is by Assumption \ref{assump:A1} (c).
Meanwhile, $| I \cap \Psi^\cB| = |I| - K$ by Assumption \ref{assump:A1} (a). 
Plugging into 
\eqref{eq:proof-init-SB-1}, this proves \eqref{eq:init-sepa-SB}.
\end{proof}

%
%
%
%
\begin{proof}[Proof of Proposition \ref{prop:I-gap-preserve}]
We will establish that for any $k =1,\ldots,n$,
\begin{equation}\label{eq:stable-deform-lambda}
|\lambda_k(t) - \lambda_k(0)| 
\le 
\frac{4 C }{ \underline{d_0}  } t, 
\quad 0 \le t \le 1.
\end{equation}
Given that this inequality holds, then by \eqref{eq:A4-cond-C}, 
\begin{equation}\label{eq:lambdakt-lambdak0-bound}
 | \lambda_k(t) - \lambda_k(0) | 
 \le 
 \frac{4 C }{ \underline{d_0}  } 
 \le
 \frac{1}{2} \Delta, 
 \quad \forall 1 \le k \le n.
 \end{equation}
This means that it is impossible for $\Delta(t) < \Delta $:
Otherwise, there exist $| \lambda_{k_1}(t) - \lambda_{k_2}(t) | < \Delta$,
where $k_1 \in I$ and $k_2 \notin I$, 
and then \eqref{eq:lambdakt-lambdak0-bound}
implies that $ | \lambda_{k_1}(0) - \lambda_{k_2}(0) | < 2 \Delta $
which contradicts the assumption that $\Delta(0) \ge 2\Delta$. 

It suffices to show \eqref{eq:stable-deform-lambda} to finish the proof.
To do so, we prove the following bound 
\begin{equation}\label{eq:bound-lambda-prime}
| \dot{\lambda}_{k} | \le \frac{4 C }{ \underline{d_0} } ,
\quad \forall k,\, \forall t.
\end{equation}
From \eqref{eq:lambda-evolve}, 
\[
| \dot{\lambda}_{k} | 
 = |\psi_{k}^{T}( \dot{W} -\lambda_k \dot{D} )\psi_{k}| 
 \le |\psi_{k}^{T} \dot{W} \psi_{k} | + |\lambda_k| | \psi_{k}^{T} \dot{D} \psi_{k} |.
\]
As $|\lambda_k |\le 1$ (Perron-Frobenius),
then
\begin{equation}\label{eq:two-term-lemma-Igap}
 | \dot{\lambda}_{k} | \le |\psi_{k}^{T} \dot{W} \psi_{k} | + | \psi_{k}^{T} \dot{D} \psi_{k} |.
 \end{equation}
 If the following claim is true, then \eqref{eq:bound-lambda-prime} follows directly from \eqref{eq:two-term-lemma-Igap}:
\begin{equation}\label{eq:bounds-1}
|\psi_l^T \dot{W} \psi_k|,\,
|\psi_l^T \dot{D} \psi_k| \le \frac{2 C }{ \underline{d_0} },
\quad \forall k,l = 1,\cdots, n,
\, \forall 0 \le t \le 1.
\end{equation}

Proof of \eqref{eq:bounds-1}:
To bound $ |\psi_l^T \dot{W} \psi_k| $,
note that $\dot{W} = E$, 
and then
\begin{equation} \label{eq:bound1-lemma-Igap}
 |\psi_{l}^{T} E \psi_{k} | 
 \le \sum_{\substack{x\in \cC, y\in \cB}} |W(x,y)| |\psi_l(x)||\psi_k(y)|  
 + \sum_{\substack{x\in \cB,  y\in \cC}} |W(x,y)| |\psi_l(x)||\psi_k(y)|.
\end{equation}
We use the following more relaxed bound
$\forall x \in \cV=\cB \cup \cC$, $\forall k$, $\forall t$, 
\begin{equation}
\label{eq:bound-psix-relaxed}
|\psi_k(x)|^2 \le \sum_{j=1}^n \psi_j(x)^2 
= \frac{1}{d(x)}
\le \frac{1}{ \underline{d_0} }, 
\end{equation}
where the last inequality relies on $d(x) \ge \underline{d_0} $ throughout time (Lemma \ref{lemma:lower-bound-d}).
The second equality relies on $D\Psi\Psi^T = I$, thus $\Psi\Psi^T = D^{-1}$.
Then \eqref{eq:bound1-lemma-Igap} continues as 
\[
 |\psi_{k}^{T} E \psi_{k} | 
 \le 2 \sum_{\substack{x\in \cC \\ y\in \cB}} |W(x,y)| \frac{1}{ \underline{d_0} } 
 = \frac{2C}{ \underline{d_0}}.
\]
For $ |\psi_l^T \dot{D} \psi_k| $, similarly,
\begin{align*}
|\psi_{l}^{T} \dot{D} \psi_{k} | 
& \le \sum_{x\in \cV} |\psi_l(x)| |\psi_k(x)| |\dot{d}(x)| \\
& = \sum_{x\in \cV} |\psi_l(x)| |\psi_k(x)| |\sum_{y} E(x,y)| \\
& \le \frac{1}{ \underline{d_0}} 
 	\sum_{x\in \cB} \sum_{y\in \cC} W(x,y)
	+ \frac{1}{ \underline{d_0} } \sum_{x\in \cC} \sum_{y\in \cB} W(x,y)  \\
& = \frac{2 C}{ \underline{d_0} },
\end{align*}
where the bound \eqref{eq:bound-psix-relaxed} is used to bound each $|\psi_l(x)|$ and $|\psi_k(x)|$
in the 2nd inequality. 

Note that while time dependence has been omitted in all the notations,
the above arguments hold throughout time $t \in [0,1]$.
\end{proof}

%
%
%
%
\begin{proof}[Proof of Proposition \ref{prop:S-evolve}]
As explained in the text, one may first establish the validity of \eqref{eq:psi-evolve}
and then verify the formula \eqref{eq:S-evolve} based on the former by  observing the cancelation of terms. 
As an alternative approach, we use the contour integral of the resolvent.

For $z \in C$ and not an eigenvalue of $P=D^{-1}W$, define 
\[
R(z) = (W - zD)^{-1}
\]
where the time dependence is omitted. 
By that $P = \Psi \Lambda \Phi^T$, $\Lambda = \text{diag}\{\lambda_1, \cdot, \lambda_n\}$, 
$\Phi = D \Psi$ and $\Psi^T\Phi = I$, 
one can verify the equivalent form of $R$ as 
\begin{equation}\label{eq:formula-R-2}
R(z) = \Psi (\Lambda - zI)^{-1}\Psi = \sum_{k=1}^n \frac{\psi_k \psi_k^T}{ \lambda_k - z}. 
\end{equation}
This means that 
\[
P_I 
=\sum_{k \in I} \psi_k \psi_k^T
= -\frac{1}{2\pi i}\oint_\Gamma R(z) dz,
\]
where the contour $\Gamma$ is such that
the eigenvalues in $I$ ($I^c$) stay inside (outside) $\Gamma$ throughout time $t$ (Figure \ref{fig:diag1}),
and such $\Gamma$ exists due to Proposition \ref{prop:I-gap-preserve}. 
Thus the above expression of $P_I$ holds for all time $t$,
and as a result
\[
\dot{P_I }
= - \frac{1}{2\pi i}\oint_\Gamma \dot{R}(z) dz.
\]
By differencing both sides of \[
(W-zD)R = I,
\]
one obtains that 
\[
\dot{R} = -R (\dot{W} - z\dot{D}) R.
\]
This means that 
\begin{align}
\dot{P_I } 
& = 
\frac{1}{2\pi i}\oint_\Gamma R(z) (\dot{W} - z\dot{D}) R(z) dz \nonumber \\
& =
 \sum_{k=1}^n \sum_{l=1}^n 
\frac{1}{2\pi i}\oint_\Gamma
\frac{\psi_k^T(\dot{W} - z\dot{D})\psi_l}{ (\lambda_k - z)(\lambda_l - z) }
 \psi_k \psi_l^T dz \nonumber \\
& = 
\sum_{k=1}^n \sum_{l=1}^n 
( \psi_k^T( \alpha_{kl} \dot{W} - \beta_{kl} \dot{D})\psi_l )
 \psi_k \psi_l^T 
 \label{eq:PI-formula-2}
\end{align}
where 
\[
\alpha_{kl} 
= \frac{1}{2\pi i}\oint_\Gamma \frac{1}{ (\lambda_k - z)(\lambda_l - z) } dz,
\quad
\beta_{kl} 
 = \frac{1}{2\pi i}\oint_\Gamma \frac{z}{ (\lambda_k - z)(\lambda_l - z) } dz.
\]
By Cauchy's integral formula, one can verify the following:

(a) When $k \in I$, $l \in I$, 
 $\alpha_{kl} = 0 $, $\beta_{kl} = 1$.

(b) When $k \in I$, $l \notin I$, 
 $\alpha_{kl} = \frac{1}{\lambda_k - \lambda_l}$, 
 $\beta_{kl} = \frac{\lambda_k}{\lambda_k - \lambda_l}$. 

(c) When $k \notin I$, $l \in I$, 
 $\alpha_{kl} = \frac{-1}{\lambda_k - \lambda_l}$, 
 $\beta_{kl} = \frac{-\lambda_l}{\lambda_k - \lambda_l}$.

(d) When $k \notin I$, $l \notin I$, 
$\alpha_{kl} = 0 $, $\beta_{kl} = 0$.

Then \eqref{eq:PI-formula-2} continues as 
\begin{align*}
\dot{P_I } 
& = 
\sum_{\substack{k\in I \\ l\in I}}
-( \psi_k^T \dot{D}\psi_l )
 \psi_k \psi_l^T 
+
\sum_{\substack{k\in I \\ l\notin I}}
\frac{ \psi_k^T( \dot{W} - \lambda_k \dot{D})\psi_l }{ \lambda_k - \lambda_l}
 \psi_k \psi_l^T  +
\sum_{\substack{k\notin I \\ l\in I}}
\frac{ \psi_k^T( - \dot{W} + \lambda_l \dot{D})\psi_l } { \lambda_k - \lambda_l}
 \psi_k \psi_l^T \\
& = 
\sum_{\substack{k\in I \\ l\in I}}
-( \psi_k^T \dot{D}\psi_l )
 \psi_k \psi_l^T   +
\sum_{\substack{k\in I \\ l\notin I}}
\frac{ \psi_k^T( \dot{W} - \lambda_k \dot{D})\psi_l }{ \lambda_k - \lambda_l}
( \psi_k \psi_l^T + \psi_l \psi_k^T ).
\end{align*}
Since $S(x) = P_I(x,x)$, the claim follows by evaluating at the entry $(x,x)$ on both sides.
\end{proof}

%
%
%
%
\begin{proof}[Proof of Theorem \ref{thm:S-sepa}]
We firstly show that 
condition (ii) implies \eqref{eq:A4-cond-C}:
Note that
\[
(\#)
	 \le \frac{\underline{d_0} (1-\varepsilon_1) }{ \overline{d_0}} 
	 \le \frac{\underline{d_0} }{ \overline{d_0}} 
	 \le 1,
\]
thus
 \[
 \log \left( 1 + \frac{1}{2}
\cdot \frac{ (\#) }{ 1+ 2 \varepsilon_1 } 
 \right)
 \le
 \log ( 1 + \frac{1}{2} (\#) ) 
 \le 
 \log \left( 1 + \frac{1}{2} \right) < 0.5.
\]
Together with
$\frac{1}{1+\frac{\Delta}{4}} < 1$,
this means that the r.h.s. of \eqref{eq:cond2-thm} is less than $0.5 \frac{\Delta}{8}$.

As a result,
under these assumptions,
Proposition \ref{prop:I-gap-preserve} and Proposition \ref{prop:S-evolve} apply.
By \eqref{eq:S-evolve}, 
\begin{align} 
& S(x,t) - S(x,0) 
 = 
\int_0^t
 \left\{ -\sum_{\substack{k\in I \\j \in I}} (\psi_{j}^{T}\dot{D}\psi_{k}) \psi_k(x)\psi_j(x)+ 2 
 \sum_{\substack{k\in I \\ j \notin I}}
 \frac{\psi_{j}^{T} (\dot{W} -\lambda_k \dot{D})\psi_{k}}{\lambda_{k}-\lambda_{j}} \psi_k(x) \psi_{j}(x)
\right\} d\tau,
\label{eq:proof-thm-1}
\end{align}
where in the integrand all the variables involving time take value at time $\tau$.

Introducing the notation
\begin{equation}\label{eq:def-barS}
\bar{S}(t):= \sup_{x\in \cV = \cB \cup \cC} S(x,t),
\end{equation}
we are going to prove the following two claims: 
For any $t \in [0,1]$,

(1)$\forall x\in \cV = \cB \cup \cC$, 
\[
|S(t,x) - S(0,x)| \le \tilde{C} \int_0^t \bar{S}(\tau) d\tau,
\quad
\tilde{C} := \left( 1+\frac{4}{\Delta} \right) \frac{2 C }{ \underline{d_0}},
\]

(2) $\bar{S}(t) \le e^{\tilde{C} t} \bar{S}(0) $.

If true, then 
\begin{align*}
& |S(t,x) - S(0,x)| 
 \le  \tilde{C} \int_0^t \bar{S}(\tau) d\tau \le \tilde{C} \int_0^t \bar{S}(0) e^{\tilde{C}\tau} d\tau  
=  \bar{S}(0) (e^{\tilde{C}t} -1).
\end{align*}
Meanwhile, by Proposition \ref{prop:init-sepa} (2),
$
\bar{S}(0) 
\le 
\frac{1}{n \underline{d_0}} \frac{K}{\delta} (1+2\varepsilon_1)
$,
thus 
\[
|S(t,x) - S(0,x)| \le (e^{\tilde{C}t} -1) \frac{K(1+2\varepsilon_1)}{n \delta \underline{d_0} }, 
\quad \forall x \in \cV.
\]
By Proposition \ref{prop:init-sepa} (1),
the initial separation on $\cB$ and $\cC$ by $S(x)$  is at least $g_0$,
so the threshold claimed in the theorem exists as long as 
\[
g_0 
\ge 
2 (e^{\tilde{C}t} -1) \frac{K(1+2\varepsilon_1)}{n \delta \underline{d_0} }.
\]
This is reduced to 
\[
2(e^{\tilde{C} t} -1) \le \frac{ (\#) }{ 1+ 2 \varepsilon_1 } ,
\]
which is guaranteed by condition (ii). 

To prove Claim (1):
By \eqref{eq:proof-thm-1}, $\forall x \in \cV$,
\[
|S(x,t) - S(x,0) | 
\le \int_0^t \text{I}(x,\tau) + \text{II}(x,\tau) + \text{III}(x,\tau) d\tau,
\]
where 
\begin{align}
\text{I}(x,\tau) 
&= |\sum_{k\in I,\, j \in I} (\psi_{j}^{T}\dot{D}\psi_{k}) \psi_k(x)\psi_j(x)| \\
\text{II}(x,\tau) 
& = 2 |
 \sum_{k\in I, \, j \notin I}
 \frac{ \psi_{j}^{T} \dot{W} \psi_{k}}{\lambda_{k}-\lambda_{j}} \psi_k(x) \psi_{j}(x)|\\
 \text{III}(x,\tau) 
& = 2 |
 \sum_{k\in I, \, j \notin I}
 \frac{ \lambda_k \psi_{j}^{T} \dot{D} \psi_{k}}{\lambda_{k}-\lambda_{j}} \psi_k(x) \psi_{j}(x)|.
\end{align}
For $\text{I}(x,\tau) $, note that 
\begin{align*}
\text{I}(x,\tau) 
& \le 
\sum_{k\in I,\, j \in I} \sum_{y \in \cV } |\dot{D}(y)| |\psi_{j}(y)| |\psi_{k}(y)| |\psi_k(x)| |\psi_j(x)| \\
& = \sum_{y \in \cV } |\dot{D}(y)| ( \sum_{k\in I} |\psi_{k}(y)| |\psi_k(x)| )^2  \\
& \le \sum_{y \in \cV } |\dot{D}(y)| ( \sum_{k\in I} |\psi_{k}(y)|^2) ( \sum_{l\in I} |\psi_l(x)| )^2 ) \\
& = \sum_{y \in \cV } |\dot{D}(y)| S(y) S(x).
\end{align*}
Utilizing the estimate that 
(similar to \eqref{eq:bound-psix-relaxed})
\[
S(x) 
= \sum_{j \in I }^n \psi_j(x)^2 
 \le 
\sum_{j=1}^n \psi_j(x)^2 
= \frac{1}{d(x)}
\le \frac{1}{ \underline{d_0}}, 
\quad \forall x \in \cV,
\]
and then
\begin{align*}
\text{I}(x,\tau) 
& \le S(x) \sum_{y \in \cV } |\dot{D}(y)| \frac{1}{ \underline{d_0} } = S(x) \frac{1}{ \underline{d_0}} \sum_{y \in \cV } |\sum_{z \in \cV} E(y,z)| 
= S(x) \frac{2C}{\underline{d_0}}.
\end{align*}
By the definition of $\bar{S}$ in \eqref{eq:def-barS},
this gives  
\begin{equation}\label{eq:bound-I}
\text{I}(x,\tau)  
\le 
	\bar{S}(\tau) \frac{2C}{\underline{d_0}}.
\end{equation}
For $\text{II}(x,\tau) $, 
by Proposition \ref{prop:I-gap-preserve}, 
in the denominator
$|\lambda_k - \lambda_j| \ge \Delta$, 
and then
\begin{align*}
\text{II}(x,\tau) 
& \le 2 
 \sum_{k\in I, \, j \notin I}
 \frac{ |\psi_{j}^{T} \dot{W} \psi_{k}|}{|\lambda_{k}-\lambda_{j}|} |\psi_k(x)| |\psi_{j}(x)| \\
& \le 
\frac{2}{\Delta} 
\sum_{k\in I, \, j \notin I}
|\psi_{j}^{T} E \psi_{k}| |\psi_k(x)| |\psi_{j}(x)| \quad \text{($\dot{W} = E$)}\\
& 
\le \frac{2}{\Delta} 
\sum_{k\in I, \, j \notin I}
\sum_{ y\in\cV, \, z\in \cV}
E(y,z) |\psi_j(y)| |\psi_k(z)| |\psi_k(x)| |\psi_{j}(x)| \\
& = \frac{2}{\Delta} 
\sum_{ y\in\cV, \, z\in \cV} E(y,z)
\sum_{k \in I} |\psi_k(z)||\psi_k(x)| 
\sum_{j \notin I} |\psi_j(y)| |\psi_{j}(x)| \\
& \le \frac{2}{\Delta} 
\sum_{ y\in\cV, \, z\in \cV} E(y,z)
(\sum_{k \in I} \psi_k(z)^2)^{1/2} (\sum_{l \in I} \psi_l(x)^2)^{1/2} \\
& ~~~~~~~~ 
(\sum_{j =1}^n \psi_j(y)^2)^{1/2} (\sum_{m=1}^n \psi_{m}(x)^2)^{1/2} \\
& \le \frac{2}{\Delta} 
\sum_{ y\in\cV, \, z\in \cV} E(y,z)
\sqrt{S(z)} \sqrt{S(x)} \frac{1}{\sqrt{d(y)}} \frac{1}{\sqrt{d(x)}}  \\
& ~~~~~~~~
\text{(definition of $S$, $\sum_{j=1}^n \psi_j(x)^2 = \frac{1}{d(x)}$)}\\
& \le \frac{2}{\Delta}
\bar{S}(\tau) \frac{1}{ \underline{d_0} }
\sum_{ y\in\cV, \, z\in \cV} E(y,z) \\
& ~~~~~~~~
 \text{(by that  $S(x), \, S(z) \le \bar{S}(\tau)$, and $d(y), \, d(x) \ge \underline{d_0}$)}
\end{align*}
which means that 
\begin{equation}\label{eq:bound-II}
\text{II}(x,\tau) 
\le 
\frac{2}{\Delta}
\bar{S}(\tau) \frac{2C}{ \underline{d_0}}.
\end{equation}
For $\text{III}(x,\tau)$, 
since $|\lambda_k| \le 1$, 
and $\dot{D}(y) = \sum_{z \in \cV} E(y,z)$,
one can show that 
\begin{equation}\label{eq:bound-III}
\text{III}(x,\tau) 
\le 
\frac{2}{\Delta}
\bar{S}(\tau) \frac{2 C}{ \underline{d_0}}
\end{equation}
using a similar argument as in bounding $\text{II}(x,\tau)$.
Putting  \eqref{eq:bound-I} \eqref{eq:bound-II} \eqref{eq:bound-III} together, one has that
\[
|S(x,t) - S(x,0) | 
\le
\int_0^t 
\bar{S}(\tau)
\frac{2 C}{\underline{d_0}} 
\left(1 + \frac{4}{\Delta}\right)
d\tau
= \int_0^t  \bar{S}(\tau) \tilde{C} d\tau,
\]
namely Claim (1).

To prove Claim (2):
Note that 
\[
\bar{S}(t) - \bar{S}(0) 
\le \sup_{x \in \cV} ( S(x,t) - S(x,0) )
\]
(suppose that $\bar{S}(t) = S( x_0 , t)$ for some $x_0$,
then $S( x_0 , t) - \bar{S}(0) \le S( x_0 , t) - S(x_0, 0)$).
Since Claim (1) holds uniformly for $x$, 
this implies that 
\[
\bar{S}(t) - \bar{S}(0) 
\le \tilde{C} \int_0^t \bar{S}(\tau) d\tau,
\]
and the claim then follows by Gronwall's inequality.
\end{proof}

\section{Further Comments}

{\it  Eigenvector selection}.
Related works have devised different methods to perform eigenvector selection to identify anomalies~\cite{miller2010subgraph,wu2013spectral}. 
 The proposed spectral embedding norm can also be used for eigenvector selection.
\cite{Mishne2018} demonstrated that it can be used to identify pixels 
which define clusters and find the embedding coordinates that best separate them from the background.
An example in the anomaly detection case is given in Appendix \ref{app:A}.

{\it Viewed as diffusion distance}. 
With $f(\lambda) = \lambda^p$, $S(x)$ can be interpreted as the (squared) {\it diffusion distance} between node $x$ and the origin at diffusion time $\frac{p}{2}$
\cite{coifman2006diffusion}. 
The diffusion distance can be interpreted as a geometric distance between two nodes 
when the affinity graph is built from data points lying on a 
manifold embedded in the ambient space,
and the distance is intrinsic to the manifold geometry and invariant to the specific embedding.
Since $\lambda_k^p \to 0$ when $p$ is large (except for $\lambda_k =1$),
the origin point is the limiting point of the diffusion map embedding.
Thus for $x$ in a sub-cluster in $\cC$, 
the weighted norm $S(x)$ with positive $p$ can be viewed as a measurement of the 
extent of metastability (the depth of the well) of the potential well associated with the sub-cluster.
 In view of the diffusion distance, under the setting of this paper, nodes in $\cB$ are very similar to one another,
and in comparison nodes in $\cC$ are distinct from those in $\cB$ (as well as from other sub-clusters in $\cC$, which is not reflected in $S$). 
A similar weighted form has also been studied in \cite{CHENG201848} for graph-based outlier detection. 
In the primary application considered in this paper, 
the leading eigenvalues are all close to 1, 
which means that the weighted form \eqref{eq:def-S-weighted} is not very different unless $p$ is large.
On the other hand setting $p$ to be large may suppress the high-index eigenvectors by small weights while they are actually the informative ones to indicate $\cC$. 
Due to these reasons, we mainly consider $f=1$ in the current paper, though the analysis directly extends. 

{\it Relation to Heat kernel signature}. 
With $f(\lambda) = e^{-(1-\lambda) t}$ from some positive $t$, 
$S(x)$ takes the form of the Heat Kernel Signature (HKS) proposed by Sun, Ovsjanikov, and Guibas~\cite{sun2009concise}.
HKS is used in the shape analysis community as a concise multiscale feature descriptor on manifolds or 3D meshes in tasks of shape matching, correspondence and retrieval. 
Note that the HKS has been used locally to identify salient features on a given shape, similar to the spectral embedding norm used for outlier detection in this paper.
The application setting for HKS focuses on Riemannian manifolds and 3D mesh,
which differs from here as we analyze the affinity matrix of a high dimensional dataset composed of clusters and background.
In particular, for reasons explained above, we mainly consider $f=1$ in the definition of $S(x)$.

{\it Indexing eigenvectors by support regions}. 
The phenomenon studied here also suggests that 
sorting by the magnitude of eigenvalues may not be the most informative way to index the eigenvectors,
a problem recently addressed in \cite{saito2018can}.
Here we study the special case where eigenvectors can be grouped by where they are mainly supported on.
In the pseudo-dynamic \eqref{eq:deformW},
the eigenvectors begin with being exactly supported on either $\cC$ or $\cB$ at $t=0$,
and as time develops this pattern is nearly preserved as long as the $\cC$-$\cB$ inter-block connections are not too strong.
The distinct support regions of eigenvectors 
appears to be irrelevant to the magnitude of the eigenvalues nor the existence of spectral gaps.
This suggests that grouping eigenvectors by their localization regions maybe a better way to arrange them in such cases.
However, one still needs to be careful with the instability of eigenvectors: 
As shown in the numerical example,
 when two eigenvalues get close in the pseudo-dynamic, 
 the associated pair of eigenvectors ``swap'' their values. 
 (The swapping may be analyzed by the differential equation \eqref{eq:psi-evolve}: 
 assuming that among all the pairs of neighboring eigenvalues 
 only one pair $(\lambda_k - \lambda_j) $ is approaching zero,
 then the dynamic of $\psi_k$ evolution is dominated by that pair, 
 which approximates a rotation among the indices $j$ and $k$.)
Our analysis in the current paper handles this by the summation in $S$ over the index group $I$, 
which makes $S$ invariant to such swaps as long as $j$ and $k$ both belong to $I$.
Generally, since eigen-crossings only happen at isolated times in the deformation dynamics \cite{kato2013perturbation}, 
these special times can be excluded. Then one can say that the eigenvectors continue to almost localize on one of the two blocks most of the time.

\section*{Acknowledgments}  
We would like to thank Ronald R. Coifman for useful discussions. 

\bibliographystyle{IEEEtran}
\bibliography{eig}

\appendix

\setcounter{table}{0}
\renewcommand{\thetable}{A.\arabic{table}}
\setcounter{figure}{0}
\renewcommand{\thefigure}{A.\arabic{figure}}

\begin{figure}[t]
\centering{
 \subfloat[] {\hspace{-6pt}
 	\includegraphics[height=0.2\linewidth]{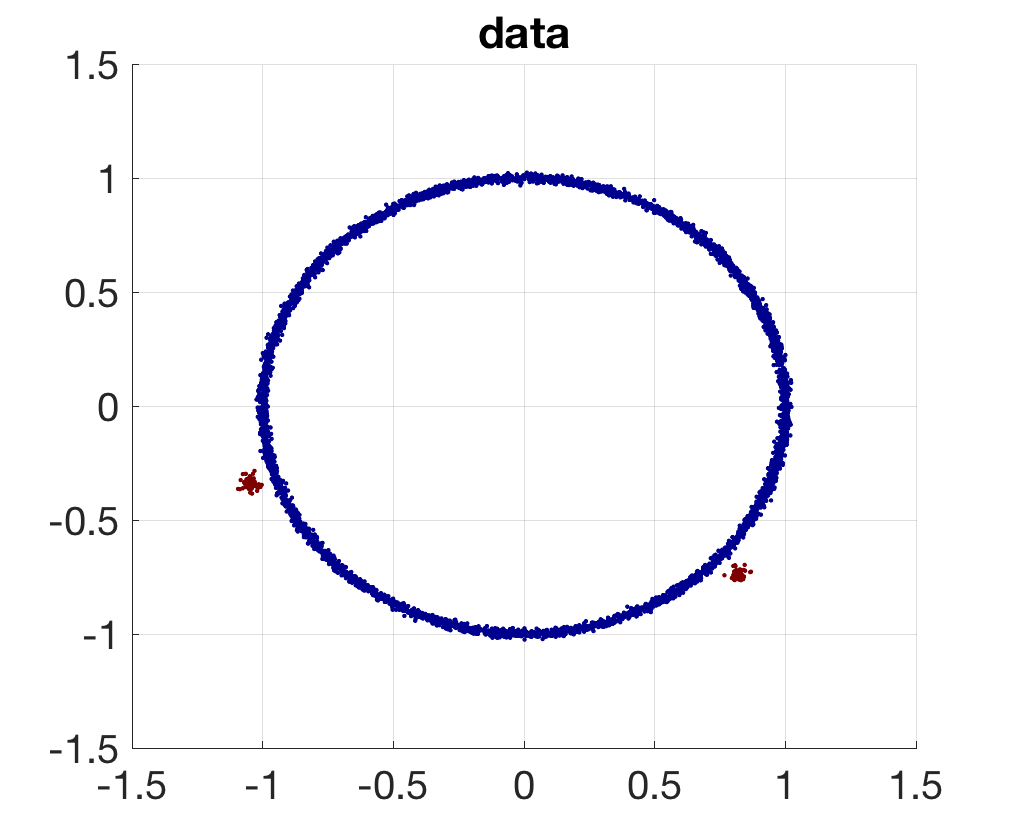} \hspace{-10pt}}
\subfloat[] {\includegraphics[height=0.2\linewidth]{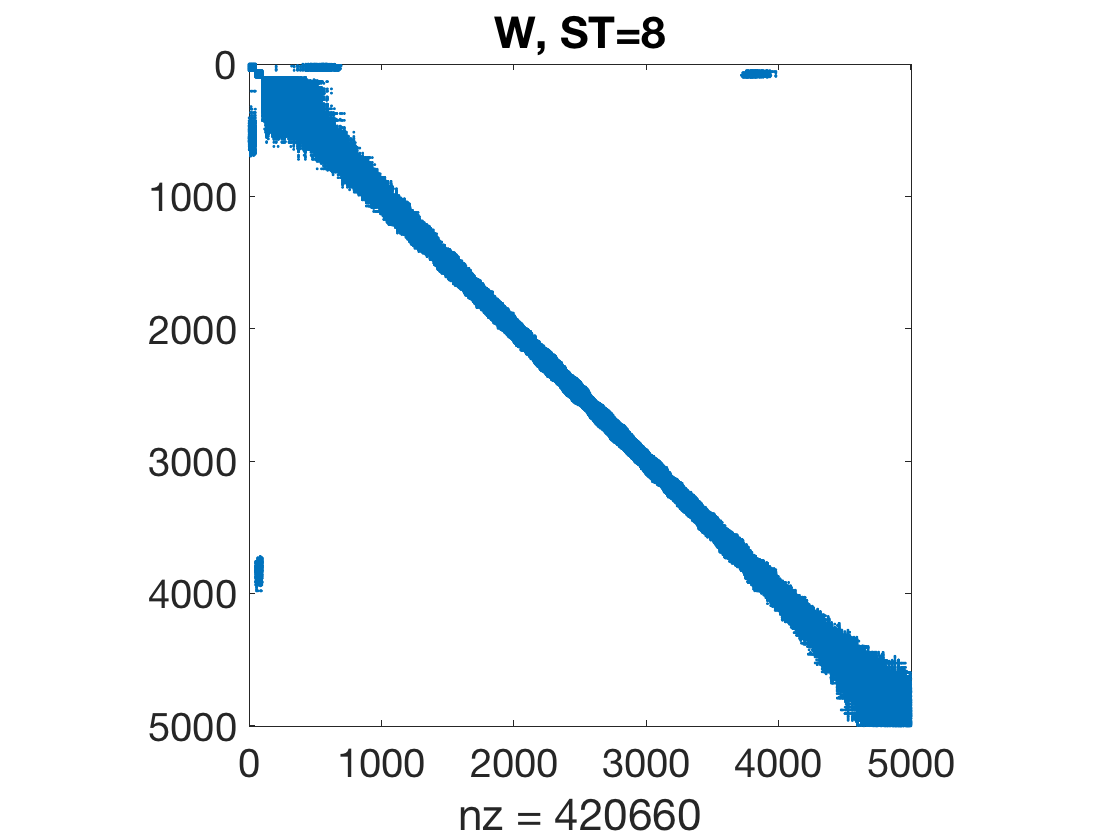}\hspace{-10pt}}
\subfloat[] {\includegraphics[height=0.2\linewidth]{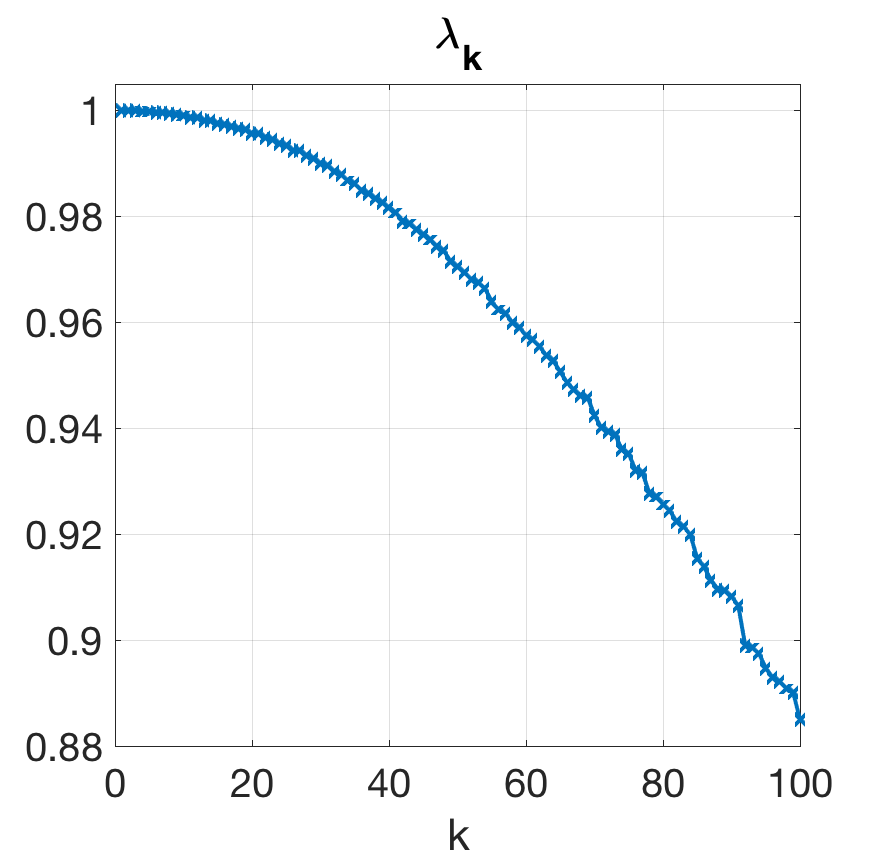}  \hspace{-6pt}}
\subfloat[]{ \includegraphics[height=0.2\linewidth]{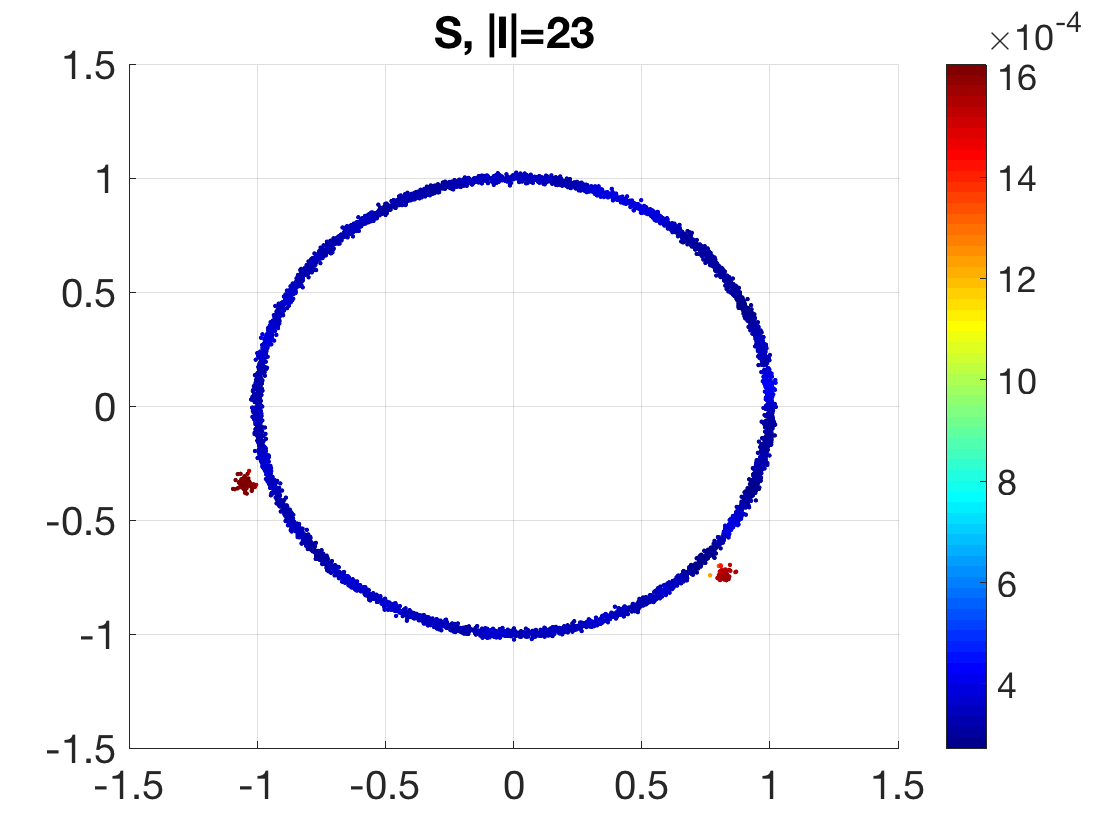} }\\
\subfloat[]{ 
\hspace{-15pt}
\includegraphics[height=0.288\linewidth]{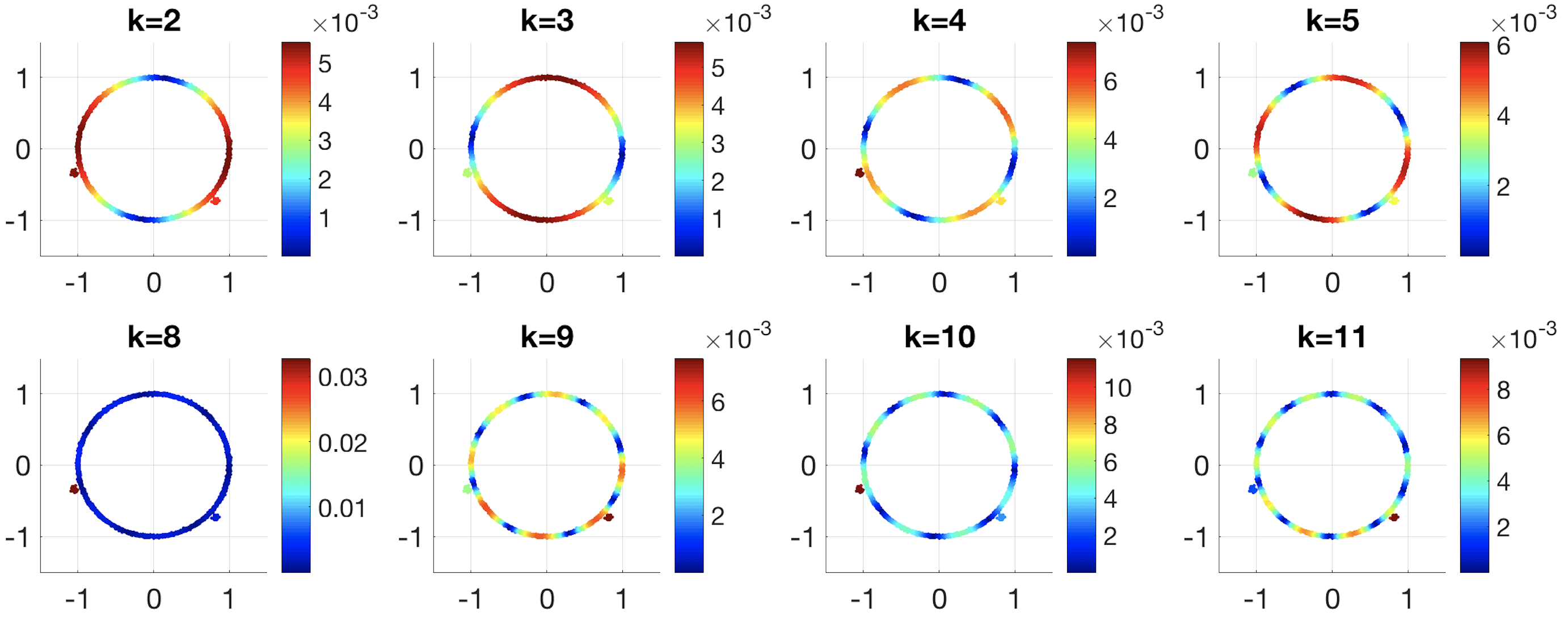} \hspace{-10pt}}
\subfloat[]{\includegraphics[height=0.29\linewidth]{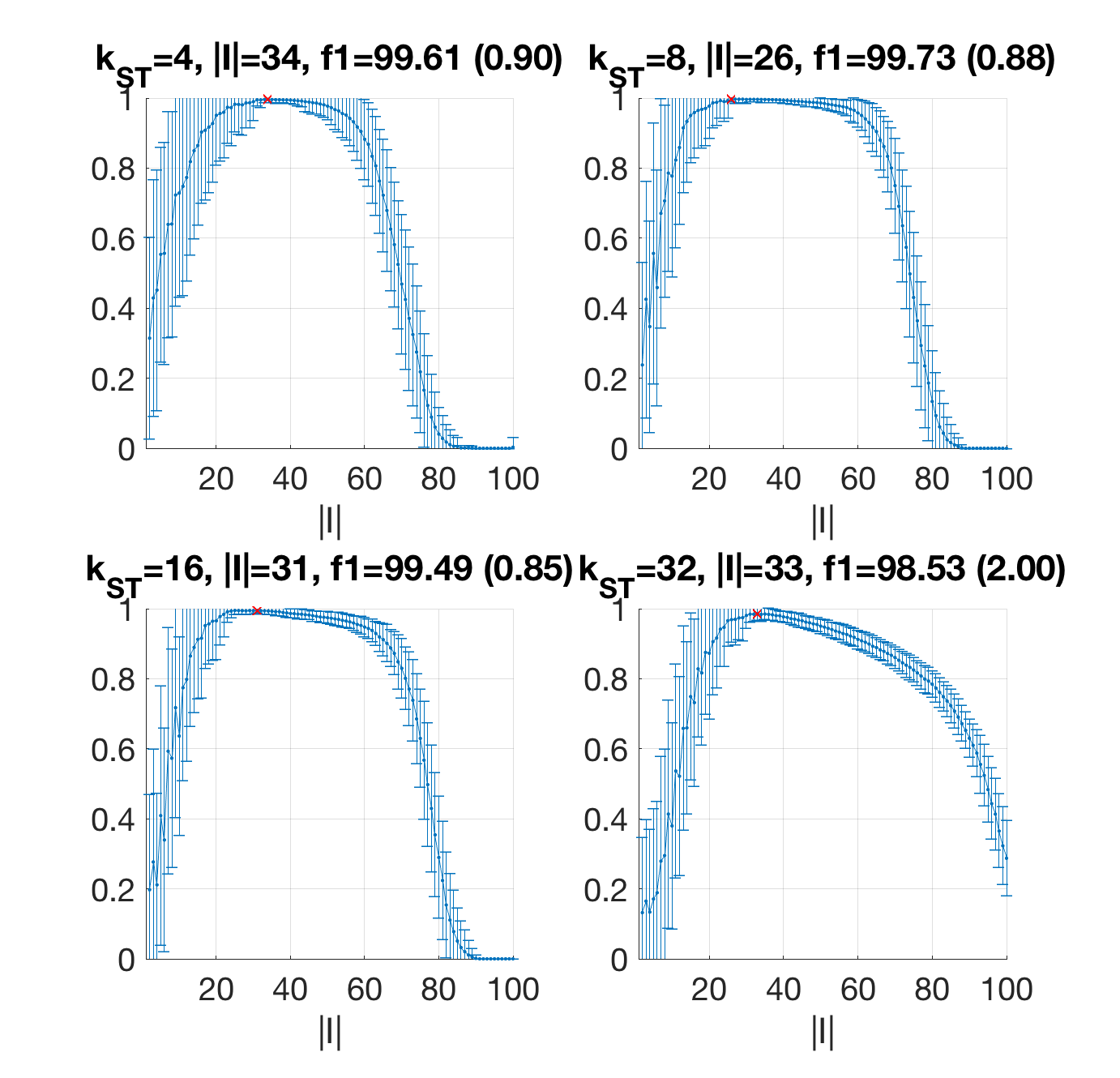}}
}
\vspace{-10pt}
\caption{ 
Same plot as Figure \ref{fig:toy2-exp}. 
$K=2$ clusters, $\delta = 0.02$.
}
\label{fig:toy1-exp}
\end{figure}

\section{Selection of eigenvectors in anomaly detection}\label{app:A}

Eigenvector selection can be used to better visualize and characterize an anomaly in the data, by finding a subspace in which it is separated from the normal data. 
Let $x_\textrm{max}=\arg\max_x S(x)$ be the pixel with the maximal embedding norm.
In Figures~\ref{fig:mine_maxInd}-\ref{fig:mine_maxInd2} we select the three eigenvectors with maximum absolute value on $x_\textrm{max}$, plotting $\vert \Psi_k(x_\textrm{max}) \vert$ on the left).
In the middle plot we color the pixels in the image according to pseudo-RGB values assigned to the three selected coordinates (right plot).
Note that for Figure~\ref{fig:mine_maxInd2}, the selected coordinates are quite deep in the spectrum: $k=59,47,48$. 

\begin{figure}[h]
\centering{\includegraphics[width=1.0\linewidth]{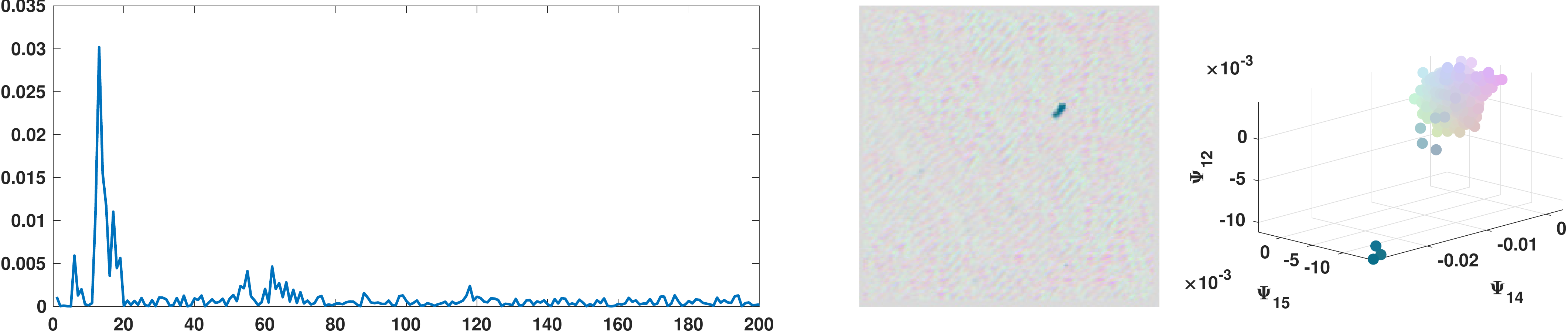}}
\caption{
Selected eigenvectors of the graph Laplacian computed from the side-scan sonar image in Figure~\ref{fig:sea-mines}(a).
(Left) 
Plot of $\vert \Psi_k(x_\textrm{max}) \vert$ where $x_\textrm{max}=\underset{x}{\arg \max} S(x)$,
and $x$-axis is the index $k$. 
(Middle) Each pixel colored according to RGB colors assigned to the three embedding coordinates, 
(Right) with maximum absolute value on $x_\textrm{max}$. 
}
\label{fig:mine_maxInd}
\end{figure}

\begin{figure}[h]
\centering{\includegraphics[width=1.0\linewidth]{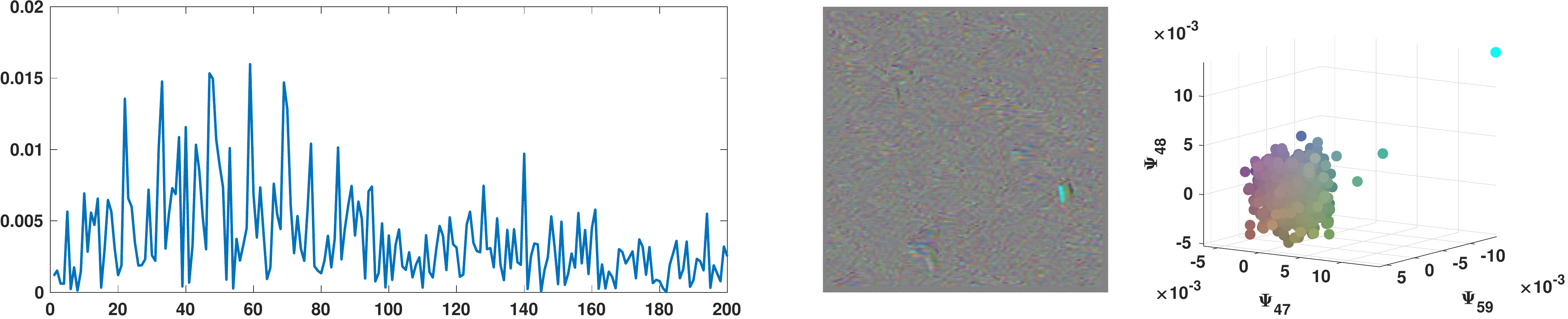}}
\caption{
Same plot as Figure \ref{fig:mine_maxInd} for the side-scan sonar image in Figure~\ref{fig:sea-mines}(c).
}
\label{fig:mine_maxInd2}
\end{figure}

\section{Derivation of \eqref{eq:lambda-evolve} and \eqref{eq:psi-evolve}}
\label{app:derive}

Recall that $\{ \lambda_k \}_{k=1}^n$ are the $n$ real-valued eigenvalues of $P$,
and $\psi_k$ is the eigenvector associated with $\lambda_k$.
By definition,
\[
W \psi_k = \lambda_k D \psi_k, 
\quad
\psi_k^T D \psi_l =\delta_{kl},
\]
where $W$, $D$, $\lambda_k$ and $\psi_k$ all depend on $t$ (the dependence is omitted in the notation)
and are differentiable over time. 
Taking derivative w.r.t $t$ of both sides, 
\begin{align}
\dot{W} \psi_k + W \dot{\psi}_k
& = \dot{\lambda}_k D \psi_k 
  + \lambda_k \dot{D} \psi_k 
  + \lambda_k D \dot{\psi}_k 
  \label{eq:dot-1}\\
0 
&= \dot{\psi}_k^T D \psi_l 
   +\psi_k^T \dot{D} \psi_l 
   +\psi_k^T D \dot{\psi}_l 
   \label{eq:dot-2}
\end{align}

Applying $\psi_k^T$ to both sides of \eqref{eq:dot-1} gives
\[
\psi_k^T \dot{W} \psi_k 
= \dot{\lambda}_k \psi_k^T D \psi_k 
  + \lambda_k \psi_k^T \dot{D} \psi_k 
= \dot{\lambda}_k  + \lambda_k \psi_k^T \dot{D} \psi_k 
\]
where in the first equality we use $W \psi_k = \lambda_k D \psi_k$. 
This proves \eqref{eq:lambda-evolve}.

To prove \eqref{eq:psi-evolve},
since $\{ \psi_k \}_k$ form a $D$-orthonormal basis of $\R^n$, let
\[
\dot{\psi}_k = \sum_{j=1}^n b_j \psi_j,
\quad
b_j = \dot{\psi}_k^T D \psi_j. 
\]
Apply $\psi_j^T$, $j \neq k$, to both sides of \eqref{eq:dot-1} gives 
\[
\psi_j^T \dot{W} \psi_k = (\lambda_k - \lambda_j) \psi_j^T D \dot{\psi}_k + \lambda_k \psi_j^T \dot{D} \psi_k,
\]
thus 
\[
 \beta_j = \frac{\psi_j^T ( \dot{W} - \lambda_k \dot{D} )\psi_k}{ \lambda_k - \lambda_j }, 
\quad j\neq k
\]
whenever $ \lambda_k - \lambda_j \neq 0 $. 
Letting $l = k$ in \eqref{eq:dot-2} gives  
\[
2 \dot{\psi}_k^T D \psi_k 
   +\psi_k^T \dot{D} \psi_k =0,  
\]
thus \[
\beta_k = - \frac{1}{2}\psi_k^T \dot{D} \psi_k.
\]
This proves \eqref{eq:psi-evolve}.

\end{document}